\documentclass[12pt]{article}

\usepackage{amssymb}
\usepackage{amsmath}
\usepackage{amsthm}
\usepackage{accents}
\usepackage{enumerate}
\usepackage{algorithm}
\usepackage{graphicx}
\usepackage{url}
\usepackage{bm}

\theoremstyle{plain}

\newtheorem{theo}{Theorem}
\newtheorem{prop}{Proposition}
\newtheorem{lemm}{Lemma}
\newtheorem{coro}{Corollary}

\newtheorem{assump}{Assumption}

\theoremstyle{definition}
\newtheorem{prob}{Problem}

\theoremstyle{definition}
\newtheorem{remark}{Remark}

\def\0{\bm{0}}
\def\1{\bm{1}}
\def\2{\bm{2}}
\def\3{\bm{3}}
\def\4{\bm{4}}
\def\5{\bm{5}}
\def\6{\bm{6}}
\def\7{\bm{7}}
\def\8{\bm{8}}
\def\9{\bm{9}}

\def\a{\bm{a}}
\def\b{\bm{b}}

\def\e{\bm{e}}
\def\f{\bm{f}}
\def\g{\bm{g}}

\def\k{\bm{k}}

\def\n{\bm{n}}

\def\p{\bm{p}}

\def\r{\bm{r}}

\def\u{\bm{u}}

\def\w{\bm{w}}
\def\x{\bm{x}}

\def\z{\bm{z}}

\def\A{\bm{A}}
\def\B{\bm{B}}
\def\C{\bm{C}}
\def\D{\bm{D}}

\def\F{\bm{F}}
\def\G{\bm{G}}

\def\I{\bm{I}}

\def\K{\bm{K}}
\def\L{\bm{L}}
\def\M{\bm{M}}
\def\N{\bm{N}}

\def\P{\bm{P}}

\def\R{\bm{R}}

\def\U{\bm{U}}
\def\V{\bm{V}}
\def\W{\bm{W}}
\def\X{\bm{X}}

\def\CC{\mathcal{C}}

\def\EC{\mathcal{E}}
\def\FC{\mathcal{F}}
\def\GC{\mathcal{G}}

\def\IC{\mathcal{I}}
\def\JC{\mathcal{J}}

\def\SC{\mathcal{S}}

\def\Real{\mbox{$\mathbb{R}$}}

\def\Natural{\mbox{$\mathbb{N}$}}
\def\SymMat{\mbox{$\mathbb{S}$}}

\def\LS{\scriptsize \bm{L}}
\def\FS{\scriptsize \bm{F}}
\def\XS{\scriptsize \bm{X}}

\def\Pib{\mbox{\bm{$\Pi$}}}
\def\zeros{\mbox{\scriptsize $\0$}}

\makeatletter
\def\widebar{\accentset{{\cc@style\underline{\mskip10mu}}}}
\def\Widebar{\accentset{{\cc@style\underline{\mskip8mu}}}}
\makeatother

\newcommand{\wb}{\widebar}
\newcommand{\wh}{\widehat}
\newcommand{\wt}{\widetilde}
\newcommand{\equivSym}{\Leftrightarrow}

\def\Primal{\mbox{$\mathbb{Q}$}}
\def\Dual{\mbox{$\mathbb{Q}_*$}}

\title{Ellipsoidal Rounding for Nonnegative Matrix Factorization  
Under Noisy Separability}
\author{Tomohiko Mizutani 
\thanks{Department of Information Systems Creation,
        Kanagawa University,
        3-27-1 Rokkakubashi, Kanagawa-ku,
        Yokohama, Kanagawa, 221-8686, Japan.
        {\tt mizutani@kanagawa-u.ac.jp}}}
\date{\today}

\begin{document}

\maketitle

\begin{abstract}
We present a numerical algorithm 
for nonnegative matrix factorization (NMF) problems under noisy separability.
An NMF problem under separability can be stated as one of 
finding all vertices of the convex hull of data points.
The research interest of this paper is 
to find the vectors as close to the vertices as possible
in a situation  in which noise is added to the data points.
Our algorithm is designed to capture the shape of the convex hull of
data points by using its enclosing ellipsoid.
We show that the algorithm has correctness and robustness properties
from theoretical and practical perspectives;
correctness here means that if the data points do not contain any noise, 
the algorithm can find the vertices of their convex hull;
robustness means that if the data points contain noise, 
the algorithm can find the near-vertices.
Finally, we apply the algorithm to document clustering, 
and report the experimental results. \bigskip    
\end{abstract}

\section{Introduction}
\label{Sec: Intro}

This paper presents a numerical algorithm 
for  nonnegative matrix factorization (NMF) problems 
under noisy separability.
The problem can be regarded as a special case of an NMF problem.
Let $\Real^{d \times m}_+$ be the set of $d$-by-$m$ nonnegative matrices, 
and $\Natural$ be the set of nonnegative integer numbers.
A nonnegative matrix is a real matrix whose elements are all nonnegative.
For a given $\A \in \Real^{d \times m}_+$ and  $r \in \Natural$, 
the nonnegative matrix factorization (NMF) problem 
is to find $\F \in \Real^{d \times r}_+$ and  $\W \in \Real^{r \times m}_+$ such that
the product $\F\W$ is as close to $\A$ as possible.
The nonnegative matrices $\F$ and $\W$ give a factorization of $\A$ of the form,
\begin{equation*}
\A = \F\W + \N,
\end{equation*}
where $\N$ is a $d$-by-$m$  matrix.
This factorization is referred to as the NMF of $\A$.

Recent studies have shown that 
NMFs are useful for tackling various problems
such as facial image analysis \cite{Lee99}, 
topic modeling \cite{Aro12b, Aro13, Din13}, 
document clustering \cite{Xu03, Sha06}, 
hyperspectral unmixing \cite{Nas05, Mia07,Gil13},
and blind source separation \cite{Cic09}.
Many algorithms have been developed 
in the context of solving such practical applications.
However, there are some drawbacks in the use of NMFs for such applications.
One of them is in the hardness of solving an NMF problem.
In fact, the problem has been shown to be NP-hard  in \cite{Vav09}.

As a remedy for the hardness of the problem, 
Arora et al.\ \cite{Aro12a} proposed to exploit the notion of separability, 
which was originally introduced by 
Donoho and Stodden in \cite{Don03} for  the uniqueness of NMF.
An NMF problem under separability becomes a tractable one.
{\it Separability} assumes that $\A \in \Real^{d \times m}_+$ can be represented as
\begin{equation}
 \A = \F \W \ \mbox{for} \ \F \in \Real^{d \times r}_+ \ 
  \mbox{and} \ \W = (\I, \K) \Pib \in \Real^{r \times m}_+,
  \label{Eq: Separability}
\end{equation}
where $\I$ is an $r$-by-$r$ identity matrix, $\K$ is an $r$-by-$(m-r)$ nonnegative matrix, 
and $\Pib$ is an $m$-by-$m$ permutation matrix.
This means that each column of $\F$ corresponds to that of $\A$ up to a scaling factor.
A matrix $\A$ is said to be a {\it separable matrix} 
if it can be represented in the form (\ref{Eq: Separability}).
In this paper, we call $\F$ the {\it basis matrix} of a separable matrix, 
and $\W$, as well as its submatrix $\K$, the {\it weight matrix}.
{\it Noisy separability} assumes that a separable matrix $\A$ 
contains a noise matrix $\N$ such that $\wt{\A} = \A + \N$,
where $\N$ is a $d$-by-$m$  matrix.
Arora et al.\ showed  that there exists an algorithm for 
finding the near-basis matrix of a noisy separable one
if the noise is small in magnitude.
Although a separability assumption restricts the fields of application for NMFs,
it is known  \cite{Aro12a, Aro12b, Aro13, Kum13, Gil13} to be  reasonable 
at least, in the contexts of document clustering, topic modeling, 
and hyperspectral unmixing.
In particular, this assumption is widely used as a pure-pixel assumption 
in hyperspectral unmixing (See, for instance, \cite{Nas05, Mia07, Gil13}).

An NMF problem under noisy separability
is to seek for the basis matrix of a noisy separable one.
The problem is formally described as follows:
\begin{prob}
 Let a data matrix $\M$ be a noisy separable matrix of size $d$-by-$m$.
 Find an index set $\IC$ with cardinality $r$ on $\{1, \ldots, m\}$
 such that $\M(\IC)$ is as close to the basis matrix $\F$ as possible.
 \label{Prob: SepNMF}
\end{prob}
Here, $\M(\IC)$ denotes a submatrix of $\M$ that
consists of every column vector with an index in $\IC$.
We call the column vector of $\M$ a {\it data point}
and that of the basis matrix $\F$ a {\it basis vector}.
An ideal algorithm for the problem should have correctness and robustness properties;
correctness here is that,
if the data matrix $\M$ is just a separable one,
the algorithm can find the basis matrix;
robustness is that,
if the data matrix $\M$ is a noisy separable one,
the algorithm can find the near-basis matrix.
A formal description of the properties is given in Section \ref{Subsec: Preliminaries}

We present a novel algorithm for Problem \ref{Prob: SepNMF}.
The main contribution of this paper is to show that 
the algorithm has correctness and robustness properties 
from theoretical and practical perspectives.
It is designed on the basis of the geometry of a separable matrix.
Under reasonable assumptions, 
the convex hull of the column vectors of a separable matrix
forms a simplex, and in particular, 
the basis vectors correspond to the vertices.
Therefore,
if all vertices of a simplex can be found, we can obtain the basis matrix
of the separable matrix.
Our algorithm uses the fact that the vertices of simplex can be found by an ellipsoid.
That is, if we draw the minimum-volume enclosing ellipsoid (MVEE) for a simplex,
the ellipsoid only touches its vertices.
More precisely, 
we give plus and minus signs to the vertices of a simplex,
and take the convex hull;
it becomes a crosspolytope having the simplex as one of the facets.
Then, the MVEE for the crosspolytope only touches the vertices of the simplex 
with plus and minus signs.

Consider Problem \ref{Prob: SepNMF} without noise.
In this case, the data matrix is just a separable one.
Our algorithm computes the MVEE for the data points
and outputs the points on the boundary of the ellipsoid.
Then, the obtained points correspond to the basis vectors 
for a separable matrix.
We show in Theorem \ref{Theo: Correctness} that the correctness property holds.
Moreover, the algorithm  works well even when the problem contains noise.
We show in Theorem \ref{Theo: Robustness} that, if the noise is lower than a certain level,
the algorithm correctly identifies the near-basis vectors for a noisy separable matrix, 
and hence, the robustness property holds.
The existing algorithms \cite{Aro12a, Bit12, Gil13a, Gil13b, Gil13, Kum13}
are formally shown to have these correctness and robustness properties.
In Section \ref{Subsec: Cmp},
our correctness and robustness properties are compared with those of the existing algorithms.

It is possible that 
noise will exceed the level that  Theorem \ref{Theo: Robustness} guarantees.
In such a situation, the MVEE for the data points may touch many points.
Hence, $r$ points need to be selected from the points on the boundary of the ellipsoid.
We make the selection by using  existing algorithms 
such as SPA \cite{Gil13} and XRAY \cite{Kum13}.
Our algorithm thus works as a preprocessor 
which filters out basis vector candidates from the data points 
and enhance the performance of existing algorithms.

We demonstrated the robustness of the algorithms to noise
through experiments with synthetic data sets.
In particular, we experimentally compared our algorithm with SPA and XRAY.
We synthetically generated data sets with various noise levels, 
and measured the robustness of an algorithm by its recovery rate.
The experimental results indicated that 
our algorithm can improve the recovery rates of SPA and XRAY.

Finally, we applied our algorithm to document clustering.
Separability for a document-term matrix means that 
each topic has an anchor word. 
An anchor word is a word which is contained in one topic 
but not contained in the other topics.
If an anchor word is found, 
it suggests the existence of its associated topic.
We conducted experiments with document corpora and 
compared the clustering performances of our algorithm and SPA.
The experimental results indicated that 
our algorithm would usually outperform SPA
and can extract more recognizable topics.

The rest of this paper is organized as follows.
Section \ref{Sec: Outline and Comparison} gives 
an outline of our algorithm and reviews related work. 
Then, the correctness and robustness properties of our algorithm
are given, and  a comparison with existing algorithms is described.
Section \ref{Sec: MVEE} reviews the formulation and algorithm of 
computing the MVEE for a set of points.
Sections \ref{Sec: Algorithm} and \ref{Sec: Implementation} 
are the main part of this paper.
We show the correctness and robustness properties of our algorithm
in Section \ref{Sec: Algorithm}, and discuss its practical implementation 
in Section \ref{Sec: Implementation}.
Section \ref{Sec: Experiment} reports the numerical experiments 
for the robustness of algorithms and document clustering.
Section \ref{Sec: Concluding} gives concluding remarks.

\subsection{Notation and Symbols} 
We use $\Real^{d \times m}$ to denote a set of real matrices of size $d$-by-$m$,
and $\Real^{d \times m}_+$ to denote a set of nonnegative matrices of $d$-by-$m$.
Let $\A \in \Real^{d \times m}$.
The symbols $\A^\top$ and $\mbox{rank}(\A)$ respectively denote 
the transposition and the rank.
The symbols $||\A||_p$ and $||\A||_F$ 
are the matrix $p$-norm and the Frobenius norm.
The symbol $\sigma_i(\A)$ is 
the $i$th largest singular value.
Let $\a_i$ be the $i$th column vector of $\A$, and 
$\IC$ be a subset of $\{1,\ldots,m\}$.
The symbol $\A(\IC)$ denotes a $d$-by-$|\IC|$ submatrix of $\A$ such that $(\a_i : i \in \IC)$.
The convex hull of all the column vectors of $\A$
is denoted by $\mbox{conv}(\A)$, and 
referred to as the convex hull of $\A$ for short.
We denote an identity matrix and a vector of all ones 
by $\I$ and $\e$, respectively.

We use $\SymMat^d$ to denote 
a set of real symmetric matrices of size $d$.
Let $\A \in \SymMat^d$.
If the matrix is positive definite, we represent it as $\A \succ \0$.
Let $\A_1 \in \SymMat^d$ and $\A_2 \in \SymMat^d$.
We denote by $\langle \A_1, \A_2 \rangle$
the Frobenius inner product of the two matrices 
which is given as the trace of matrix $\A_1\A_2$.

We use a MATLAB-like notation.
Let $\A_1 \in \Real^{d \times m_1}$ and 
$\A_2 \in \Real^{d \times m_2}$.
We denote by $(\A_1, \A_2)$ the horizontal concatenation 
of the two matrices, which is a $d$-by-$(m_1 + m_2)$ matrix.
Let $\A_1 \in \Real^{d_1 \times m}$ and $\A_2 \in \Real^{d_2 \times m}$.
We denote by $(\A_1; \A_2)$  the vertical concatenation 
of the two matrices, and it is a matrix of the form,
\begin{equation*}
 \left(
 \begin{array}{c}
  \A_1 \\
  \A_2
 \end{array}
\right) \in \Real^{(d_1 + d_2) \times m}.
\end{equation*}
Let $\A$ be a $d$-by-$m$ rectangular diagonal matrix having 
diagonal elements $a_{1}, \ldots, a_{t}$ where $t = \min\{d,m\}$.
We use $\mbox{diag}(a_1, \ldots, a_t)$ to denote  the matrix.

\section{Outline of Proposed Algorithm and Comparison with Existing Algorithms}
\label{Sec: Outline and Comparison}

Here, we formally describe the properties mentioned in Section \ref{Sec: Intro}
that an algorithm is expected to have, 
and also describe the assumptions we place on Problem \ref{Prob: SepNMF}.
Next, we give a geometric interpretation of a separable matrix 
under these assumptions, and then, outline the proposed algorithm.
After reviewing the related work,
we describe the correctness and robustness properties of our algorithm
and compare with those of the existing algorithms.

\subsection{Preliminaries}
\label{Subsec: Preliminaries}

Consider Problem \ref{Prob: SepNMF} whose data matrix $\M$ 
is a noisy separable one of the form $\A + \N$.
Here, $\A$ is a separable matrix of (\ref{Eq: Separability}) and $\N$ is a noise matrix.
We can rewrite it as
\begin{eqnarray}
 \M & = & \A + \N  \nonumber \\
    & = & \F(\I, \K)\Pib + \N   \nonumber \\
   & = & (\F + \N^{(1)}, \F\K + \N^{(2)}  )\Pib  
  \label{Eq: Rep of noisy separable matrix}
\end{eqnarray}
where 
$\N^{(1)}$ and $\N^{(2)}$ are $d$-by-$r$ and $d$-by-$\ell$ submatrices of $\N$ 
such that $\N \Pib^{-1} = (\N^{(1)}, \N^{(2)})$.
Hereinafter, we use the notation $\ell$ to denote $m-r$.
The goal of Problem \ref{Prob: SepNMF} is to identify an index set $\IC$ such that $\M(\IC) = \F + \N^{(1)}$.

As  mentioned in Section \ref{Sec: Intro}, 
it is ideal that 
an algorithm for Problem \ref{Prob: SepNMF} has correctness and robustness properties.
These properties are formally described as follows:
\begin{itemize}
 \item {\bf Correctness.} 
       If the data matrix $\M$ does not contain a noise matrix $\N$
       and is just a separable matrix,
       the algorithm returns an index set $\IC$ such that $\M(\IC) = \F$.
	    
 \item {\bf Robustness.}
       If the data matrix $\M$ contains a noise matrix $\N$
       and is a noisy separable matrix such that
       $||\N||_p < \epsilon$,
       the algorithm  returns an index set $\IC$ such that
       $||\M(\IC) -\F||_p < \tau \epsilon$ 
       for some constant real number $\tau$.
\end{itemize}
In particular, the robustness property has  $\tau = 1$,
if an algorithm can identify
an index set $\IC$ such that $\M(\IC) = \F + \N^{(1)}$
where $\F$ and $\N^{(1)}$ are of (\ref{Eq: Rep of noisy separable matrix})
since $||\M(\IC) -\F||_p = ||\N^{(1)}||_p < \epsilon$.

In the design of the algorithm, 
some assumptions are usually placed on a separable matrix.
Our algorithm uses Assumption \ref{Assump: Separability}.
\begin{assump}  \label{Assump: Separability}
 A separable matrix $\A$ of (\ref{Eq: Separability}) consists of 
 an basis matrix $\F$ and a weight matrix $\W$ satisfying the following conditions.
\begin{enumerate}[\ref{Assump: Separability}-a)]
 \item  Every column vector of weight matrix $\W$ has unit 1-norm.
 \item  The basis matrix $\F$ has full column rank.  
\end{enumerate}
\end{assump}
Assumption 1-a can be invoked without loss of generality.
If the $i$th column of $\W$ is zero, 
so is the $i$th column of $\A$.
Therefore, we can construct a smaller separable matrix having $\W$ with no zero column.
Also, since we have 
\begin{equation*}
 \A = \F \W \equivSym  \A \D = \F \W \D,
\end{equation*}
every column of $\W$ can have unit 1-norm.
Here, $\D$ denotes a diagonal matrix having 
the $(i, i)$th diagonal element $d_{ii} = 1 / ||\w_i||_1$.

The same assumption 
is used by the algorithm in \cite{Gil13}.
We may get the feeling that 1-b is strong.
The algorithms in \cite{Aro12a, Bit12, Gil13a, Gil13b, Kum13} 
instead assume {\it simpliciality}, 
wherein no column vector of $\F$ can be represented as a convex hull of 
the remaining vectors of $\F$.
Although 1-b is a stronger assumption, 
it still seems reasonable for Problem \ref{Prob: SepNMF} 
from the standpoint of practical application.
This is because, in such cases, 
it is less common for the column vectors of the basis matrix $\F$ to be linearly dependent.

\subsection{Outline of Proposed Algorithm}

Let us take a look at Problem \ref{Prob: SepNMF} from a geometric point of view.
For simplicity, consider the noiseless case first.
Here, a data matrix is just a separable matrix $\A$.
Separability implies that $\A$ has a factorization of the form (\ref{Eq: Separability}).
Under Assumption \ref{Assump: Separability},
$\mbox{conv}(\A)$ becomes an $(r-1)$-dimensional simplex in $\Real^d$.
The left part of Figure~\ref{Fig: geoNMF} visualizes a separable data matrix.
The white points are data points, and the black ones are basis vectors.
The key observation is that the basis vectors $\f_1, \ldots, \f_r$ of $\A$
correspond to the vertices of $\mbox{conv}(\A)$.
This is due to separability.
Therefore, if all vertices of $\mbox{conv}(\A)$ can be found,
we can obtain the basis matrix $\F$ of $\A$.
This is not hard task, 
and we can design an efficient algorithm for doing it.
But, if noise is added to a separable matrix, the task becomes hard.
Let us suppose that 
the data matrix of Problem \ref{Prob: SepNMF} 
is a noisy separable matrix $\wt{\A}$ of the form $\A + \N$.
The vertices of $\mbox{conv}(\wt{\A})$
do not necessarily match the basis vectors $\f_1, \ldots, \f_r$ of $\A$.
The right part of Figure~\ref{Fig: geoNMF} visualizes a noisy separable data matrix.
This is the main reason why it is hard to identify the basis matrix from noisy separable one.

\begin{figure}[h]
\begin{center}
 \includegraphics[width=0.75\columnwidth]{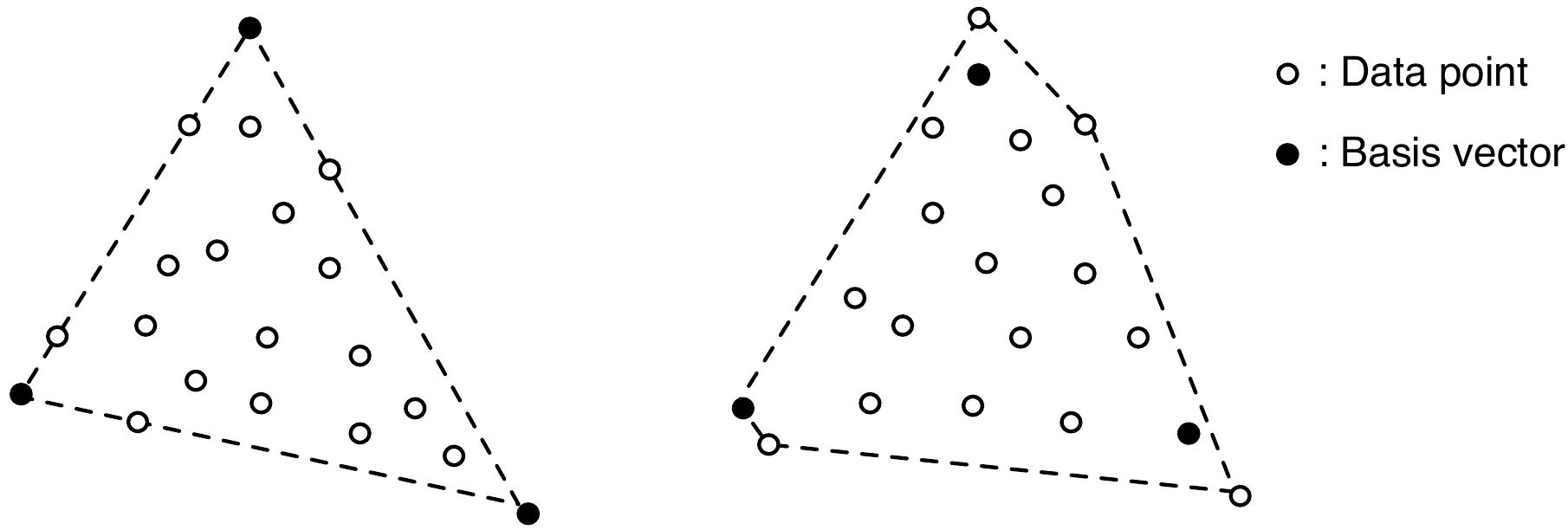}
 \caption{Convex hull of a separable data matrix with $r=3$ under Assumption \ref{Assump: Separability}.
 (Left) Noiseless case.  (Right) Noisy case.}
 \label{Fig: geoNMF}
\end{center}
\end{figure}

Our algorithm is designed on the basis of 
Proposition \ref{Prop: active points of simplex} in Section \ref{Subsec: Analysis};
it states that all vertices of a simplex can be found by using an ellipsoid.
We here describe the proposition from a geometric point of view.
Consider an $(r-1)$-dimensional simplex $\Delta$ in $\Real^{r}$.
Let $\g_1, \ldots, \g_r \in \Real^{r}$ be the vertices of $\Delta$, and 
$\b_1, \ldots, \b_\ell \in \Real^{r}$ be the points in $\Delta$.
We draw the MVEE centered at the origin 
for a set $\SC = \{\pm \g_1, \ldots, \pm \g_r, \pm \b_1, \ldots, \pm \b_\ell\}$.
Then, the proposition says
that the ellipsoid  only touches the points $\pm \g_1, \ldots, \pm \g_r$ 
among all the points in $\SC$.
Therefore, the vertices  of $\Delta$  can be found
by checking whether the points in $\SC$ lie on the boundary of ellipsoid.
We should mention that 
the convex hull of the points in $\SC$ becomes a full-dimensional 
crosspolytope in $\Real^r$.
Figure \ref{Fig: erAlg} illustrates the MVEE for a crosspolytope
in $\Real^3$.

\begin{figure}[h]
 \begin{center}
 \includegraphics[width=0.4\columnwidth]{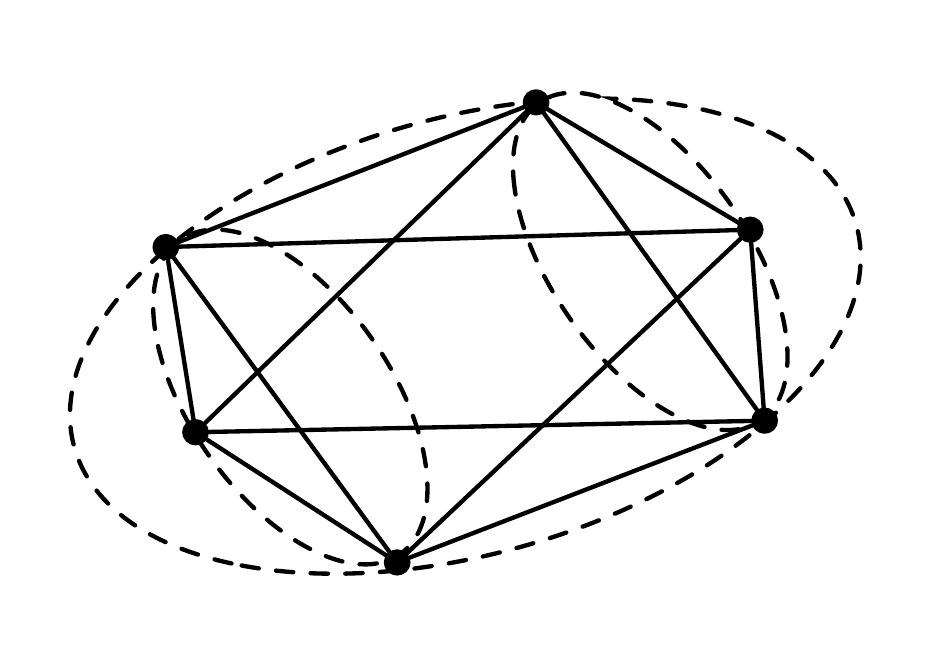}
 \caption{Minimum-volume enclosing ellipsoid for a full-dimensional 
  crosspolytope in $\Real^3$.}
 \label{Fig: erAlg}
 \end{center}
\end{figure}

Under Assumption \ref{Assump: Separability},
the convex hull of a separable matrix $\A$ 
becomes an $(r-1)$-dimensional simplex in $\Real^d$.
Therefore, we rotate and embed the simplex in $\Real^r$
by using an orthogonal transformation.
Such a transformation can be obtained by singular value decomposition (SVD) of $\A$.

Now let us outline our algorithm for Problem \ref{Prob: SepNMF}.
In this description, 
we assume for simplicity 
that the data matrix is a separable one  $\A \in \Real^{d \times m}_+$.
First, the algorithm constructs an orthogonal transformation 
through the SVD of $\A$.
By applying the transformation,
it transforms $\A$ into a matrix $\P \in \Real^{r \times m}$ such that
the $\mbox{conv}(\P)$ is an $(r-1)$-dimensional simplex in $\Real^r$.
Next, it draws the MVEE centered at the origin
for a set $\SC = \{ \pm \p_1, \ldots, \pm \p_m\}$, 
where $\p_1, \ldots, \p_m$ are the column vectors of $\P$, 
and outputs $r$ points lying on the ellipsoid.

We call the algorithm {\it ellipsoidal rounding}, abbreviated as ER.
The main computational costs of ER are in computing 
the SVD of $\A$ and the MVEE for $\SC$.
The MVEE computation can be formulated as 
a tractable convex optimization problem with $m$ variables.
A polynomial-time algorithm exists, and 
it is also known that a hybrid of the interior-point algorithm and cutting plane algorithm 
works efficiently in practice.

In later sections, we will see that 
ER algorithm works well even if noise is added.
In particular, we show 
that ER correctly identifies the near-basis vectors of a noisy separable matrix
if the noise is smaller than some level.
We consider a situation in which the noise exceeds that level.
In such a situation,
the shape of crosspolytope formed by the data points
is considerably perturbed by the noise, and 
it is possible that the MVEE touches many points.
We thus need to select $r$ points from the points 
on the boundary of the ellipsoid.
In this paper, we perform existing algorithms such as 
SPA \cite{Gil13} and XRAY \cite{Kum13} to make the selection.
Hence, ER works as a preprocessor 
which filters out basis vector candidates from data points
and enhances the performance of existing algorithms.

\subsection{Related Work} \label{Subsec: Related work}

First, we will review the algorithms for NMF of general nonnegative matrices.
There are an enormous number of studies.
A commonly used approach is to formulate it as a nonconvex optimization problem 
and compute the local solution.
Let $\A$ be a $d$-by-$m$ nonnegative matrix, and
consider an optimization problem with matrix variables 
$\F \in \Real^{d \times r}$ and $\W \in \Real^{r \times m}$,
\begin{equation*}
 \mbox{minimize} \ ||\F\W - \A ||_F^2 \ 
  \mbox{subject to} \ \F \ge \0 \ \mbox{and} \  \W \ge \0.
\end{equation*}
This is an intractable nonconvex optimization problem, 
and in fact, it was shown in \cite{Vav09} to be NP-hard.
Therefore, the research target is in how to compute 
the local solution efficiently.
It is popular to use the block coordinate descent (BCD) algorithm 
for this purpose.
The algorithm solves the problem by alternately fixing the variables $\F$ and $\W$.
The problem obtained by fixing either of $\F$ and $\W$ 
becomes a convex optimization problem.
The existing studies propose to use, for instance, 
the projected gradient algorithm \cite{Lin07} and its variant \cite{Lee00}, 
active set algorithm \cite{Kim08a, Kim11}, and projected quasi-Newton algorithm \cite{Gon12}.
It is reported that 
the BCD algorithm shows good performance on average in computing NMFs.
However, its performance depends on how we choose the initial point 
for starting the algorithm.
We refer the reader to \cite{Kim14} for a survey on the algorithms for NMF.

Next, we will survey the algorithms that work on noisy separable matrices.
Four types of algorithm can be found:
\begin{itemize}
 \item {\bf AGKM \cite{Aro12a}.} 
       The algorithm constructs $r$ sets of data points such that 
       all of the basis vectors are contained 
       in the union of the sets and each set has one basis vector.
       The construction entails solving
       $m$ linear programming (LP) problems with $m-1$ variables.
       Then, it chooses one element from each set, and  outputs them.

 \item {\bf Hottopixx \cite{Bit12, Gil13a, Gil13b}.}
       Let $\A$ be a separable matrix of the form $\F(\I,\K)\Pib$.
       Consider a matrix $\C$ such that
       \begin{equation*}
	 \C =  \Pib^{-1} 
	 \left(
	 \begin{array}{cc}
	  \I & \K \\
	  \0 & \0 
	 \end{array}
       \right) \Pib \in \Real^{m \times m}.
       \end{equation*}
       It satisfies $\A = \A\C$, and also, 
       if the diagonal element is one, 
       the position of its diagonal element 
       indicates the index of basis vector in $\A$.
       The algorithm models $\C$ as the variable of an LP problem.
       It entails solving a single LP problem with $m^2$ variables.
       
 \item {\bf SPA \cite{Gil13}.}
       Let $\A$ be a separable matrix of size $d$-by-$m$, 
       and $\SC$ be the set of the column vectors of $\A$.
       The algorithm is based on the following observation.
       Under Assumption \ref{Assump: Separability},
       the maximum of a convex function over 
       the elements in $\SC$ is 
       attained at the vertex of $\mbox{conv}(\A)$.
       The algorithm finds one element $\a$ in $\SC$ that maximizes 
       a convex function, and then,
       projects all elements in $\SC$ into the orthogonal space to $\a$.
       This procedure is repeated until $r$ elements are found.

 \item {\bf XRAY \cite{Kum13}.}
       The algorithm has a similar spirit as SPA, 
       but it uses a linear function instead of a convex one.
       Let $\A$ be a separable matrix of size $d$-by-$m$
       and $\SC$ be the set of the column vectors of $\A$.
       Let $\IC_k$ be the index set 
       obtained after the $k$th iteration.
       This is a subset of $\{1, \ldots,m\}$ with cardinality $k$.
       In the $(k+1)$th iteration, 
       it computes a residual matrix $\R = \A(\IC_k)\X^* - \A$, where
       \begin{equation*}
       \X^* = \arg \min_{\XS \ge \zeros}||\A(\IC_k)\X - \A||_2^2,
       \end{equation*}
       and picks up one of the column vectors $\r_i$ of $\R$.
       Then, it finds one element from $\SC$ which maximizes 
       a linear function having $\r_i$ as the normal vector.
       Finally, $\IC_k$ is updated by adding the index of the obtained element.
       This procedure is repeated until $r$ indices are found.
       The performance of XRAY depends on how we select the column vector of
       the residual matrix $\R$ for making the linear function.
       Several ways of selection, 
       called ``rand'', ``max'', ``dist'' and ``greedy'', 
       have been proposed by the authors.
\end{itemize}
The next section describes the properties of these algorithms.

\subsection{Comparison with Existing Algorithm}
\label{Subsec: Cmp}
We compare the correctness and robustness properties
of ER with those of AGKM, Hottopixx, SPA, and XRAY.
ER is shown to have the two properties
in Theorems \ref{Theo: Correctness} and \ref{Theo: Robustness}.
In particular, our robustness property in Theorem \ref{Theo: Robustness}
states that 
ER correctly identifies the near-basis matrix of a noisy separable one $\wt{\A}$,
and a robustness property with $\tau = 1$ holds if we set 
\begin{equation}
\epsilon = \frac{\sigma (1-\mu)}{4} 
 \label{Eq: epsilon of ER}
\end{equation}
and $p=2$ under Assumption  \ref{Assump: Separability}.
Here,  $\sigma$ is the minimum singular value of 
the basis matrix $\F$ of a separable one $\A$ in the $\wt{\A}$, 
i.e., $\sigma = \sigma_r(\F)$, and 
$\mu$ is $\mu(\K)$:
\begin{equation} \label{Eq: Mu}
 \mu(\K) = \max_{i=1,\ldots,\ell} ||\k_i||_2
\end{equation}
for a weight matrix $\K$ of $\A$.
Under Assumption 1-a,
we have $\mu \le 1$, and in particular, 
equality holds if and only if $\k_i$ has only one nonzero element.

All four of the existing algorithms have been shown 
to have a correctness property, 
whereas every one except XRAY has a robustness property.
Hottopixx is the most similar to ER.
The authors of \cite{Bit12} showed that 
it has the correctness and robustness with $\tau = 1$ properties
if one sets
\begin{equation}
\epsilon = \frac{\alpha  \min\{d_0, \alpha\}}{9(r+1)}
 \label{Eq: epsilon of Hot}
\end{equation}
and $p=1$ under simpliciality and other assumptions.
Here, $\alpha$ and $d_0$ are as follows.
$\alpha$ is the minimum value of $\delta_{\FS}(j)$ for $j=1,\ldots,r$,
where $\delta_{\FS}(j)$ denotes an $\ell_1$-distance between 
the $j$th column vector $\f_j$ of $\F$ and 
the convex hull of the remaining column vectors in $\F$.
$d_0$ is the minimum value of $||\a_i- \f_j||_1$
for every $i$ such that $\a_i$ is not a basis vector, 
and every $j = 1, \ldots, r$.
The robustness of Hottopixx is further 
analyzed in \cite{Gil13a, Gil13b}.

It can be interpreted that
the $\epsilon$ of ER (\ref{Eq: epsilon of ER}) is given by 
the multiplication of two parameters representing 
flatness and closeness of a given data matrix
since $\sigma$ measures the flatness of the convex hull of data points, and 
$1-\mu$ measures the closeness between basis vectors and data points.
Intuitively, we may say that
an algorithm becomes sensitive to noise when a data matrix has the following features;
one is that the convex hull of data points is close to a flat shape, 
and another is that there are data points close to basis vectors.
The $\epsilon$ of (\ref{Eq: epsilon of ER}) well matches the intuition.
We see a similar structure in the $\epsilon$ of Hottopixx (\ref{Eq: epsilon of Hot})
since $\alpha$ and $d_0$ respectively measure
the flatness and closeness of a given data.

Compared with Hottopixx, 
the $\epsilon$ of ER (\ref{Eq: epsilon of ER}) does not contain $1/r$, 
and hence, it does not decrease as  $r$ increases.
However,  Assumption 1-b of ER 
is stronger than the simpliciality of Hottopixx.
In a practical implementation,
ER can handle a large matrix, 
while Hottopixx may have limitations on the size of the matrix
it can handle.
Hottopixx entails solving an LP problem with $m^2$ variables.
In the NMFs arising in applications, $m$ tends to be a large number.
Although an LP is tractable, 
it becomes harder to solve as the size increases.
Through experiments,
we assessed the performance of Hottopixx with the CPLEX LP solver.
The experiments showed that 
the algorithm had out of memory issues
when $m$ exceeded 2,000 with $d=100$.
The authors of \cite{Bit12} proposed a parallel implementation 
to resolve these computational issues.

AGKM and SPA were shown in \cite{Aro12a} and \cite{Gil13}
to have a robustness property with $\tau \ge 1$ for some $\epsilon$.
In practical implementations,
SPA and XRAY are scalable to the problem size
and experimentally show good robustness. 
Section \ref{Sec: Experiment} reports
a numerical comparison of ER with SPA and XRAY.

\section{Review of Formulation and Algorithm for MVEE Computation} \label{Sec: MVEE}

We review the formulation 
for computing the MVEE for a set of points, 
and survey the existing algorithms for the computation.

First of all, 
let us recall the terminology related to an ellipsoid.
An ellipsoid in $\Real^d$ is defined as 
a set   $\EC(\L, \z) = \{ \x \in \Real^d : (\x - \z)^{\top}\L (\x - \z) \le 1\}$ 
for a positive definite matrix $\L$ of size $d$ and a vector $\z \in \Real^d$.
Here, $\L$ determines the shape of the ellipsoid and $\z$ is the center.
Let $\x$ be a point in an ellipsoid $\EC(\L, \z)$.
If the point $\x$ satisfies the equality $(\x - \z)^{\top}\L (\x - \z) = 1$, 
we call it an {\it active point} of the ellipsoid.
In other words, an active point is  one lying on the boundary of the ellipsoid.

The volume of the ellipsoid  is  given as $c(d) / \sqrt{\det \L}$,
where $c(d)$ represents the volume of a unit ball in $\Real^d$ 
and it is a real number depending on the dimension $d$.
ER algorithm considers $d$-dimensional ellipsoids 
containing a set $\SC$ of points in $\Real^d$, and 
in particular,  finds the minimum volume ellipsoid 
centered at the origin.
In this paper, such an ellipsoid is referred to as an origin-centered MVEE for short.

Now, we are ready to describe a formulation 
for computing the origin-centered MVEE for a set of points.
For $m$ points $\p_1, \ldots, \p_m \in \Real^d$, 
let $\SC = \{\pm \p_1, \ldots, \pm \p_m \}$.
The computation of the origin-centered MVEE for $\SC$ is formulated as 
\begin{equation*}
 \begin{array}{lll}
  \Primal(\SC): & \mbox{minimize}    & -\log \det \L, \\
                & \mbox{subject to} & \langle  \p_i \p_i^{\top}, \L \rangle \le 1, \quad i = 1, \ldots, m, \\
                &     & \L \succ \0,
 \end{array}
\end{equation*}
where the matrix $\L$ of size $d$ is the decision variable.
The optimal solution $\L^*$ of $\Primal$ gives
the origin-centered MVEE for $\SC$  as $\EC(\L^*) = \{\x : \x^\top \L^* \x \le 1 \}$.
We here introduce some terminology. 
An active point of $\EC(\L^*)$
is  a vector $\p_i \in \Real^d$ satisfying $\p_i^\top \L^* \p_i = 1$.
We call $\p_i$  an {\it active point of $\Primal(\SC)$}, 
and  the index $i$ of $\p_i$ an {\it active index of $\Primal(\SC)$}.
The ellipsoid $\EC(\L^*)$ is centrally symmetric, and 
if a vector $\p_i$ is an active point, so is $-\p_i$.
The dual of $\Primal$ reads 
\begin{equation*}
 \begin{array}{lll}
  \Dual(\SC): & \mbox{maximize}    &  \log \det \Omega(\u), \\
              & \mbox{subject to}   &  \e^{\top} \u = 1, \\
              &                &  \u \ge \0,
 \end{array}
\end{equation*}
where the vector $\u$ is the decision variable.
Here, $\Omega: \Real^m \rightarrow \SymMat^{d}$ is a linear function given as
$\Omega(\u) = \sum_{i=1}^{m} \p_i \p_i^{\top} u_i$;
equivalently, $\Omega(\u) = \P \mbox{diag}(\u) \P^{\top}$ 
for $\P =(\p_1, \ldots, \p_m) \in \Real^{d \times m}$.
It follows from the Karush-Kuhn-Tucker (KKT) conditions for these problems 
that the optimal solution $\L^*$ of $\Primal$ is represented by 
$\frac{1}{d} \Omega(\u^*)^{-1}$ for the optimal solution $\u^*$ of $\Dual$.
We make the following assumption to 
ensure the existence of an optimal solution of $\Primal$.
\begin{assump} \label{Assump: MVEE}
 $\mbox{rank}(\P) = d$ for $\P = (\p_1, \ldots, \p_m) \in \Real^{d \times m}$.
\end{assump}

Later, the KKT conditions will play an important role in
our discussion of the active points of $\Primal$.
Here though, we will describe the conditions:
$\L^* \in \SymMat^{d}$ is an optimal solution  for $\Primal$ 
and $\z^* \in \Real^{m}$ is the associated Lagrange multiplier vector if and only if
there exist $\L^* \in \SymMat^{d}$ and $\z^* \in \Real^{m}$ such that
 \begin{subequations}
  \label{Eq: KKT}
   \begin{eqnarray}
 & & -(\L^*)^{-1} +  \Omega(\z^*) = \0, \label{Eq1: KKT} \\
 & &  z_i^* ( \langle \p_i \p_i^{\top}, \L^* \rangle - 1) = 0, \quad i = 1, \ldots, m, \label{Eq2: KKT} \\
 & &  \langle \p_i \p_i^{\top}, \L^* \rangle \le 1, \quad  i = 1, \ldots, m, \label{Eq3: KKT} \\
 & &  \L^* \succ \0, \label{Eq4: KKT} \\
 & &   z_i^* \ge 0, \quad i = 1, \ldots, m. \label{Eq5: KKT}
   \end{eqnarray}
 \end{subequations}

Many algorithms have been proposed for solving problems $\Primal$ and $\Dual$.
These can be categorized into mainly two types:
conditional gradient algorithms (also referred to as Frank-Wolfe algorithms)
and interior-point algorithms.
Below, we survey the studies on these two algorithms.

Khachiyan in \cite{Kha96} proposed  a barycentric coordinate descent algorithm, 
which  can be interpreted as a conditional gradient algorithm.
He showed that the algorithm has a polynomial-time iteration complexity.
Several researchers investigated and revised Khachiyan's algorithm.
Paper \cite{Kum05} showed that 
the iteration complexity of Khachiyan's algorithm can be slightly reduced 
if it starts from a well-selected initial point.
Papers \cite{Tod07} and \cite{Ahi08} incorporated 
a step called as a Wolfe's away-step.
The revised algorithm was shown to 
have a polynomial-time iteration complexity and a linear convergence rate.

A dual interior-point algorithm was given in \cite{Van98}.
A primal-dual interior-point algorithm was given in \cite{Toh99a},
and numerical experiments showed that 
this algorithm is efficient and can provide accurate solutions.
A practical algorithm was designed in \cite{Sun04} for solving large-scale problems.
In particular, 
a hybrid of the interior-point algorithm and cutting plane algorithm was shown to be 
efficient in numerical experiments.
For instance, paper \cite{Sun04} reported that 
the hybrid algorithm can solve problems 
with $d=30$ and $m=30,000$ in under 30 seconds on a personal computer.
Paper \cite{Tsu07} considered generalized forms of $\Primal$ and $\Dual$
and showed that a primal-dual interior-point algorithm 
for the generalized forms has a polynomial-time iteration complexity.

Next, let us discuss the complexity of these 
two sorts of algorithms for $\Primal$ and $\Dual$.
In each iteration, 
the arithmetic operations of
the conditional gradient algorithms are 
less than those of the interior-point algorithms.
Each iteration of a conditional gradient algorithm \cite{Kha96, Kum05, Tod07, Ahi08}
requires $O(md)$ arithmetic operations.
On the other hand, assuming that 
the number of data points $m$ is sufficiently larger than the dimension of data points $d$,
the main complexity of interior-point algorithms \cite{Van98, Toh99a} 
comes from solving an $m$-by-$m$ system of linear equations in each iteration.
The solution serves as the search direction for the next iteration.
Solving these linear equations requires $O(m^3)$ arithmetic operations.
In practice, 
the number of iterations of conditional gradient algorithms
is much larger than that of interior-point algorithms.
As the paper \cite{Ahi08} reports, 
conditional gradient algorithms take several thousands iterations 
to solve  problems 
such that $d$ runs from $10$ to $30$ and $m$ from $10,000$ to $30,000$.
On the other hand, as paper \cite{Sun04} reports, 
interior-point algorithms usually terminate after several dozen iterations 
and provide accurate solutions.

One of the concerns about interior-point algorithms is the computational cost 
of each iteration.
It is possible to reduce the cost considerably 
by using a cutting plane strategy.
A hybrid of interior-point algorithm and cutting plane algorithm
has an advantage over conditional gradient algorithms.
In fact,  paper \cite{Ahi08} reports that
the hybrid algorithm is faster than 
the conditional gradient algorithms and works well even on large problems.
Therefore, we use the hybrid algorithm to solve $\Primal$
in our practical implementation of ER.
The details are  in Section \ref{Subsec: Cutting plane}.

Here, it should be mentioned that this paper uses
a terminology ``cutting plane strategy'' for 
what other papers \cite{Sun04, Ahi08} have called 
the ``active set strategy'', 
since it might be confused with ``active set algorithm'' 
for solving a nonnegative least square problem.

 \section{Description and Analysis of the Algorithm} 
 \label{Sec: Algorithm}
The ER algorithm is presented below.
Throughout of this paper,
we use  the notation $\Natural$ to denote a set of nonnegative integer numbers.
\begin{algorithm}
 \noindent 
 \caption{Ellipsoidal Rounding (ER) for  Problem \ref{Prob: SepNMF}}
 \label{Alg: ER}
 \textbf{Input:}  $\M \in \Real^{d \times m}_+$ and $r \in \Natural$.  \\
 \textbf{Output:} $\IC$.
 \begin{enumerate}
  \item[\textbf{1:}] 
	       Compute the SVD of $\M$, and construct the reduced matrix 
	       $\P \in \Real^{r \times m}$  associated with $r$.

  \item[\textbf{2:}] 
	       Let $\SC = \{\pm \p_1, \ldots, \pm \p_m\}$ 
	       for the column vectors $\p_1, \ldots, \p_m$ of $\P$.
	       Solve $\Primal(\SC)$, and construct the active index set $\IC$.

 \end{enumerate}
\end{algorithm}

Step 1 needs to be explained in detail.
Let $\M$ be a noisy separable matrix of size $d$-by-$m$.
In general, the $\M$ is a full-rank due to the existence of a noise matrix.
However, the rank  is  close to $r$ 
when the amount of noise is small, and in particular, 
it is $r$ in the noiseless case.
Accordingly, 
we construct a low-rank approximation matrix to $\M$ 
and reduce the redundancy in the space 
spanned by the column vectors of $\M$.

We use an SVD for the construction of the low-rank approximation matrix.
The SVD of $\M$ gives a decomposition of the form,
\begin{equation*}
\M = \U\Sigma\V^{\top}.
\end{equation*}
Here, 
$\U$ and $\V$ are 
$d$-by-$d$ and $m$-by-$m$ orthogonal matrices, respectively.
In this paper, we call the $\U$ a {\it left orthogonal matrix} of the SVD of $\M$.
Let $t = \min\{d,m\}$.
$\Sigma$ is 
a rectangular diagonal matrix consisting of the singular values 
$\sigma_1, \ldots, \sigma_t$ of $\M$, and it is of the form, 
\begin{equation*}
 \Sigma = \mbox{diag}(\sigma_1, \ldots, \sigma_t) \in \Real^{d \times m} 
\end{equation*}
with $\sigma_1 \ge \cdots \ge \sigma_t \ge 0$.
By choosing the top $r$ singular values while setting the others to $0$ in $\Sigma$,
we construct
\begin{equation*}
 \Sigma^r = \mbox{diag}(\sigma_1, \ldots, \sigma_r, 0, \ldots, 0) 
  \in \Real^{d \times m}
\end{equation*}
and let 
\begin{equation*}
\M^r = \U \Sigma^r \V^{\top}.
\end{equation*}
$\M^r$ is the best rank-$r$ approximation to $\M$
as measured by the matrix 2-norm and  satisfies
$||\M - \M^r||_2 = \sigma_{r+1}$ 
(see, for instance, Theorem 2.5.3 of \cite{Gol96}).
By applying the left orthogonal matrix $\U^\top$ to $\M^r$, we have
\begin{equation*}
\U^{\top} \M^r = \left(
 \begin{array}{c}
  \P \\
  \0
 \end{array}
\right) \in \Real^{d \times m},
\end{equation*}
where $\P$ is an $r$-by-$m$ matrix with $\mbox{rank}(\P) = r$.
We call such a matrix $\P$ 
a {\it reduced matrix of $\M$ associated with $r$}.
Since Assumption \ref{Assump: MVEE} holds for the $\P$, 
it is possible to  perform an MVEE computation for a set of the column 
vectors.

\subsection{Correctness for a Separable Matrix} \label{Subsec: Analysis}
We analyze the correctness property of Algorithm \ref{Alg: ER}.
Let $\A$ be a separable matrix of size $d$-by-$m$.
Assume that Assumption \ref{Assump: Separability} holds for $\A$.
We run Algorithm \ref{Alg: ER} for $(\A, \mbox{rank}(\A))$.
Step 1 computes the reduced matrix $\P$ of $\A$.
Since $r = \mbox{rank}(\A)$,
we have $\A=\A^r$, where 
$\A^r$ is  the best rank-$r$ approximation matrix  to $\A$.
Let $\U \in \Real^{d \times d}$ be the left orthogonal matrix of the SVD of $\A$. 
The reduced matrix $\P \in \Real^{r \times m}$ of $\A$ is obtained as
\begin{eqnarray}
 \left(
 \begin{array}{c}
  \P \\
  \0
 \end{array}
\right)
 &=& \U^\top \A  \nonumber \\
 &=& \U^\top \F(\I, \K) \Pib.
  \label{Eq: Rep of P}
\end{eqnarray}
From the above, we see that
\begin{equation} \label{Eq: Relation of W and G}
 \U^\top \F = 
  \left(
 \begin{array}{c}
  \G \\
  \0
 \end{array}
\right) \in \Real^{d \times m}, \ \mbox{where} \ \G \in \Real^{r \times r}.
\end{equation}
Here, we have $\mbox{rank}(\G) = r$
since $\mbox{rank}(\F) = r$ by Assumption \ref{Assump: Separability}-b 
and $\U$ is an orthogonal matrix.
By using $\G$, we rewrite $\P$ as 
\begin{equation*}
\P  = (\G, \G\K) \Pib.
\end{equation*}
From Assumption \ref{Assump: Separability}-a, 
the column vectors $\k_i$ of the weight matrix $\K \in \Real^{r \times \ell}$ 
satisfy the conditions
\begin{equation} \label{Eq: K cond}
 ||\k_i||_1=1 \  \mbox{and} \ \k_i \ge \0, \quad i = 1, \ldots, \ell.
\end{equation}

In Step 2, we collect the column vectors of $\P$ and construct a set $\SC$ of them.
Let $\B = \G\K$, and 
let $\g_j$  and $\b_i$ be the column vector of $\G$ and  $\B$, respectively.
$\SC$ is a set of vectors $\pm \g_1, \ldots, \pm \g_r, \pm \b_1, \ldots, \pm \b_\ell$.
The following proposition guarantees that
the active points of $\Primal(\SC)$ are $\g_1, \ldots, \g_r$.
We can see from (\ref{Eq: Rep of P}) and (\ref{Eq: Relation of W and G}) that
the index set of the column vectors of $\G$ is identical to that of of $\F$.
Hence, the basis matrix $\F$ of a separable one $\A$ can be obtained 
by finding the active points of $\Primal(\SC)$.
 \begin{prop} \label{Prop: active points of simplex}
  Let $\G \in \Real^{r \times r}$ and $\B = \G\K \in \Real^{r \times \ell}$ 
  for $\K \in \Real^{r \times \ell}$.
  For the column vectors $\g_j$ and $\b_i$ of $\G$ and $\B$, respectively,
  let $\SC = \{\pm \g_1, \ldots \pm \g_r, \pm \b_1, \ldots, \pm \b_\ell\}$.
  Suppose that $\mbox{rank}(\G) = r$ and $\K$ satisfies the condition (\ref{Eq: K cond}).
  Then, the active point set of $\Primal(\SC)$ is  $\{\g_1, \ldots, \g_r\}$.
\end{prop}
\begin{proof}
 We show that an optimal solution $\L^*$ of $\Primal(\SC)$ is
 $(\G\G^\top)^{-1}$ and its associated Lagrange multiplier $\z^*$ is  $(\e; \0)$,
 where $\e$ is an $r$-dimensional all-ones vector 
 and  $\0$ is an $\ell$-dimensional zero vector.
 Here, the Lagrange multipliers are one for the constraints 
  $\langle \g_j\g_j^{\top}, \L \rangle \le 1$, and 
 these are zero  for 
 $\langle \b_i\b_i^{\top}, \L \rangle \le 1$.

 Since $\G$ is nonsingular,
 the inverse of $\G\G^\top$  exists and it is positive definite.
 Now we check that $\L^* = (\G\G^{\top})^{-1}$ and $\z^* = (\e; \0)$ 
 satisfy the KKT conditions  (\ref{Eq: KKT}) for the problem.
 It was already seen that the conditions (\ref{Eq1: KKT}), (\ref{Eq4: KKT}), 
 and (\ref{Eq5: KKT}) are satisfied.
 For the remaining conditions, 
 we have 
 \begin{equation}
   \langle \g_j\g_j^{\top}, (\G\G^{\top})^{-1} \rangle 
    = (\G^{\top} (\G\G^{\top})^{-1}\G)_{jj} 
    =  1 \label{Eq: g const}
 \end{equation}
 and 
 \begin{eqnarray}
  \langle \b_i\b_i^{\top}, (\G\G^{\top})^{-1} \rangle  
  & = & (\B^{\top} (\G\G^{\top})^{-1}\B)_{ii}  \nonumber  \\
  & = & (\K^{\top} \G^{\top} (\G\G^{\top})^{-1}\G\K)_{ii} \nonumber\\
  & = & \k_i^{\top}\k_i \nonumber \\
  & \le & ||\k_i||_1^2 = 1. \label{Eq: b const}
 \end{eqnarray}
 Here, $(\cdot)_{ii}$ for a matrix denotes the $(i,i)$th element of the matrix.
 The inequality in (\ref{Eq: b const}) follows from  condition (\ref{Eq: K cond}).
 Also, the Lagrange multipliers are zero for 
 the inequality constraints 
 $\langle \b_i\b_i^{\top}, (\G\G^{\top})^{-1} \rangle \le 1$.
 Thus, conditions (\ref{Eq2: KKT}) and (\ref{Eq3: KKT}) are satisfied.
 Accordingly,  $(\G\G^{\top})^{-1}$  is an optimal solution of $\Primal(\SC)$.
 
 We can see from (\ref{Eq: g const}) that $\g_1, \ldots, \g_r$ are the 
 active points of the problem.
 Moreover, we may have equality in (\ref{Eq: b const}).
 In fact, equality holds if  and only if
 $\k_i$ has only one nonzero element.
 For such  $\k_i$, $\b_i = \G\k_i$ coincides with some vector in 
 $\g_1, \ldots, \g_r$.
\end{proof}

From the above discussion, 
we can immediately notice that this proposition 
holds if for a matrix $\K \in \Real^{r \times \ell}$,
the column vectors $\k_i$ satisfy 
\begin{equation} \label{Eq: K relaxed cond}
||\k_i||_2 < 1, \quad i = 1, \ldots, m.
\end{equation}
Note that in contrast with condition (\ref{Eq: K cond}), 
this condition does not require the matrix to be nonnegative.

\begin{coro} \label{coro: active points of simplex under relaxed cond}
 Proposition \ref{Prop: active points of simplex} holds 
 even if we suppose that 
 $\K \in \Real^{r \times \ell}$ satisfies
 condition  (\ref{Eq: K relaxed cond}), instead of condition (\ref{Eq: K cond}).
\end{coro}

Note that 
this  corollary is used to show the robustness of 
Algorithm \ref{Alg: ER} on a noisy separable matrix.
The correctness of Algorithm \ref{Alg: ER} 
for a separable matrix
follows from the above discussion and Proposition \ref{Prop: active points of simplex}.
\begin{theo}
 Let $\A$ be a separable matrix.
 Assume that Assumption \ref{Eq: Separability} holds for $\A$.
 Then, Algorithm \ref{Alg: ER} for $(\A, \mbox{rank}(\A))$ 
 returns an index set $\IC$ such that $\A(\IC) = \F$.
 \label{Theo: Correctness}
\end{theo}

\subsection{Robustness for a Noisy Separable Matrix}
Next, we analyze the robustness property of Algorithm \ref{Alg: ER}.
Let $\A$ be a separable matrix  of size $d$-by-$m$.
Assume that Assumption \ref{Assump: Separability} holds for  $\A$.
Let $\wt{\A}$ be a noisy separable matrix of the form $\A + \N$.
We run Algorithm \ref{Alg: ER} for $(\wt{\A}, \mbox{rank}(\A))$.
Step 1 computes the reduced matrix $\P$ of $\wt{\A}$.
Let $\U \in \Real^{d \times d}$ be the left orthogonal matrix of the SVD of $\wt{\A}$, 
and  $\wt{\A}^r$ be the best rank-$r$ approximation matrix to $\wt{\A}$.
We denote  the residual matrix $\wt{\A} - \wt{\A}^r$ by $\wt{\A}^r_\diamond$.
For the reduced matrix $\P$ of $\wt{\A}$, we have 
\begin{subequations}
\begin{eqnarray}
 \left(
 \begin{array}{c}
  \P \\
  \0
 \end{array}
\right) 
 &=& \U^\top \wt{\A}^r 
 \label{Eq: EquRep of P (a)}  \\
 &=& \U^\top(\wt{\A} - \wt{\A}^{r}_{\diamond}) 
  \label{Eq: EquRep of P (b)} \\
 &=& \U^\top(\A + \N - \wt{\A}^{r}_{\diamond}) 
  \label{Eq: EquRep of P (c)} \\
 &=& \U^\top(\A + \wb{\N})   
  \label{Eq: EquRep of P (d)} \\
 &=& \U^\top( (\F,\F\K)\Pib + \wb{\N})  
  \label{Eq: EquRep of P (e)} \\
 &=& \U^\top( \F + \wb{\N}^{(1)}, \F\K + \wb{\N}^{(2)})\Pib  
  \label{Eq: EquRep of P (f)}\\
 &=& \U^\top( \wh{\F}, \wh{\F}\K + \wh{\N})\Pib.
  \label{Eq: EquRep of P (g)}
\end{eqnarray}
\end{subequations}
The following notation is used in the above: 
$\wb{\N} = \N - \wt{\A}^r_\diamond$ in (\ref{Eq: EquRep of P (d)});
$\wb{\N}^{(1)}$ and $\wb{\N}^{(2)}$ in (\ref{Eq: EquRep of P (f)}) are
the $d$-by-$r$ and $d$-by-$\ell$ submatrices of $\wb{\N}$ such that
$\wb{\N} \Pib^{-1} = (\wb{\N}^{(1)}, \wb{\N}^{(2)})$;
$\wh{\F} = \F + \wb{\N}^{(1)}$ and 
$\wh{\N} = -\wb{\N}^{(1)}\K + \wb{\N}^{(2)}$ in (\ref{Eq: EquRep of P (g)}).
This  implies that 
\begin{equation} \label{Eq: HatG}
 \U^\top\wh{\F} = 
 \left(
 \begin{array}{c}
  \wh{\G}\\ 
  \0
 \end{array}
\right), \ \mbox{where} \ \wh{\G} \in \Real^{r \times r},
\end{equation}
and 
\begin{equation} \label{Eq: R}
 \U^\top\wh{\N} = 
 \left(
 \begin{array}{c}
  \R\\ 
  \0
 \end{array}
\right),
\ \mbox{where} \  \R \in \Real^{r \times \ell}.
\end{equation}
Hence, we can rewrite $\P$ as 
\begin{equation*}
 \P = (\wh{\G}, \wh{\G}\K + \R) \Pib.
\end{equation*}
$\wt{\A}$ is  represented by (\ref{Eq: Rep of noisy separable matrix}) as
\begin{equation*}
 \wt{\A}  =  (\wt{\F}, \wt{\F}\K + \wt{\N})\Pib,
\end{equation*}
where $\wt{\F}$ and $\wt{\N}$ denote 
$\F + \N^{(1)}$ and $-\N^{(1)} \K + \N^{(2)}$, respectively.
From  (\ref{Eq: EquRep of P (b)}), we have
\begin{equation*}
 \left(
 \begin{array}{c}
  (\wh{\G}, \wh{\G}\K + \R) \Pib \\
  \0
 \end{array}
\right) 
 = \U^\top((\wt{\F}, \wt{\F}\K + \wt{\N})\Pib - \wt{\A}^{r}_{\diamond}).
  \label{Eq: Rel of G and W}
\end{equation*}
Therefore, 
the index set of the column vectors of $\wh{\G}$ is identical to
that of $\wt{\F}$.
If all the column vectors of $\wh{\G}$ are found in $\P$,
we can identify $\wt{\F}$ hidden in $\wt{\A}$.

In Step 2, we collect the column vectors of $\P$ 
and construct a set $\SC$ of them.
Let 
\begin{equation}
 \wh{\B} = \wh{\G}\K + \R,
  \label{Eq: HatB}
\end{equation} 
and let 
$\wh{\g}_j$ and $\wh{\b}_i$ respectively be the column vectors 
of $\wh{\G}$ and $\wh{\B}$.
$\SC$ is a set of vectors 
$\pm \wh{\g}_1, \ldots, \pm \wh{\g}_r, \pm \wh{\b}_1,\ldots, \pm \wh{\b}_\ell$.
We can see from Corollary \ref{coro: active points of simplex under relaxed cond} that, 
if $\mbox{rank}(\wh{\G}) = r$ and 
$\wh{\b}_i$ is written as $\wh{\b}_i = \wh{\G} \wh{\k}_i$ by using $\wh{\k}_i \in \Real^{r}$
with $||\wh{\k}_i||_2 < 1$, 
the active points of $\Primal(\SC)$ are given as 
the column vectors $\wh{\g}_1, \ldots, \wh{\g}_r$ of $\wh{\G}$.
Below, we examine the amount of noise $\N$
such that the conditions of 
Corollary \ref{coro: active points of simplex under relaxed cond} still hold.
 \begin{lemm} \label{Lemm: Singular value of perturbed A}
 Let $\wt{\A} = \A + \N \in \Real^{d \times m}$.
 Then, $|\sigma_i(\wt{\A}) - \sigma_i(\A)| \le ||\N||_2$ 
  for each $i=1, \ldots, t$ where $t=\min\{d,m\}$.
\end{lemm}
\begin{proof}
See Corollary 8.6.2 of \cite{Gol96}.
\end{proof}
\begin{lemm} \label{Lemm: Noise size}
 Let $n = ||\N||_2$ and $\mu = \mu(\K)$.
 \begin{enumerate}[\ref{Lemm: Noise size}-a)]
  \item The matrix $\wb{\N}$ of (\ref{Eq: EquRep of P (d)}) satisfies $||\wb{\N}||_2 \le 2n$.
	
  \item The column vectors $\r_i$ of matrix $\R$ of (\ref{Eq: R})  
	satisfy $ ||\r_i||_2 \le 2n(\mu+1)$ for $i=1, \ldots, m$.

  \item  The singular values of  matrix $\wh{\G}$ of (\ref{Eq: HatG})
	 satisfy $ |\sigma_i(\wh{\G}) - \sigma_i(\F) | \le 2n$ for $i = 1, \ldots, r$.
 \end{enumerate}
\end{lemm}
\begin{proof}
 \ref{Lemm: Noise size}-a) \ 
 Since $\wb{\N} = \N - \wt{\A}^{r}_\diamond$,
 \begin{equation*}
  ||\wb{\N}||_2 \le || \N||_2 + ||\wt{\A}^r_\diamond||_2.
 \end{equation*}
 We have $||\wt{\A}^r_\diamond||_2 \le n$
 since $||\wt{\A}^r_\diamond||_2 = \sigma_{r+1}(\wt{\A})$ and 
 from Lemma \ref{Lemm: Singular value of perturbed A},
 $|\sigma_{r+1}(\wt{\A}) - \sigma_{r+1}(\A)| \le n$.
 Therefore, $||\wb{\N}||_2 \le 2n$.

 \ref{Lemm: Noise size}-b)  \
 Let $\wh{\n}_i$ be the column vector  of the matrix $\wh{\N}$ of (\ref{Eq: EquRep of P (g)}).
 Since $\U^\top \wh{\n}_i = (\r_i; \0)$ for an orthogonal matrix $\U$, 
we have $||\wh{\n}_i||_2 = ||\r_i||_2$.
Therefore, we will evaluate $||\wh{\n}_i||_2$.
 Let $\k_i$  and $\wb{\n}^{(2)}_i$ be the  column vectors of $\K$ and $\wb{\N}^{(2)}$, 
 respectively.
 Then, $\wh{\n}_i$ can be represented as 
 $ - \wb{\N}^{(1)} \k_i + \wb{\n}^{(2)}_i$.
 Thus, by Lemma \ref{Lemm: Noise size}-a, we have
 \begin{equation*}
  ||\r_i||_2 = ||\wh{\n}_i||_2 
   \le ||\wb{\N}^{(1)}||_2 ||\k_i||_2 + ||\wb{\n}^{(2)}_i||_2 
  \le  2n(\mu+1).
 \end{equation*}

 \ref{Lemm: Noise size}-c)  \
 Since $\U^\top \wh{\F} = (\wh{\G}; \0)$ for an orthogonal matrix $\U$, 
 the singular values of $\wh{\F}$ and $\wh{\G}$ are identical.
 Also, since $\wh{\F} = \F + \wb{\N}^{(1)}$ and 
 Lemma \ref{Lemm: Singular value of perturbed A}, we have
 \begin{eqnarray*}
  |\sigma_i(\wh{\G}) - \sigma_i(\F)| =  |\sigma_i(\wh{\F}) - \sigma_i(\F)| 
   \le ||\wb{\N}^{(1)}||_2 
   \le 2n. 
 \end{eqnarray*}
\end{proof}

The following lemma ensures that 
the conditions of Corollary \ref{coro: active points of simplex under relaxed cond} 
hold if the amount of noise is smaller than a certain level.
\begin{lemm}\label{Lemm: Check cond}
 Let $\wh{\G}$ be the matrix of (\ref{Eq: HatG}), and 
 let $\wh{\b}_i$ be the column vector of $\wh{\B}$ of 
 (\ref{Eq: HatB}).
 Suppose that $||\N||_2 < \epsilon$ for 
 $\epsilon = \frac{1}{4}\sigma(1-\mu)$ 
 where $\sigma = \sigma_r(\F) $ and $\mu = \mu(\K)$.
 Then,
 \begin{enumerate}[\ref{Lemm: Check cond}-a)]
  \item $\mbox{rank}(\wh{\G}) = r$.
  \item $\wh{\b}_i$ is represented as 
	$\wh{\G} \wh{\k}_i = \wh{\b}_i$ by using $\wh{\k}_i$ such that $||\wh{\k}_i||_2 < 1$.
 \end{enumerate}
\end{lemm}
In the proof below, $n$ denotes $||\N||_2$. 
\begin{proof}
 \ref{Lemm: Check cond}-a)  \ 
 From Lemma \ref{Lemm: Noise size}-c, the minimum singular value 
 of $\wh{\G}$ satisfies
 \begin{eqnarray*}
  \sigma_r(\wh{\G}) 
  &\ge& \sigma - 2n \\
  &>& \sigma - 2\epsilon  
  = \frac{1}{2}\sigma(1 +  \mu) > 0. \\
 \end{eqnarray*}
 The final inequality follows 
 from $\sigma > 0$ due to Assumption \ref{Assump: Separability}-b.
 Hence, we have $\mbox{rank}(\wh{\G}) = r$.

 \ref{Lemm: Check cond}-b) \
 Let $\k_i$ and $\r_i$ be 
 the column vectors of $\K$ and $\R$, respectively.
 Then, we have $\wh{\b}_i = \wh{\G}\k_i + \r_i$.
 Since Lemma \ref{Lemm: Check cond}-a guarantees that 
 $\wh{\G}$ has an inverse, 
 it can be represented as $\wh{\b}_i = \wh{\G}\wh{\k}_i$ 
 by $\wh{\k}_i = \k_i + \wh{\G}^{-1}\r_i$.
 It follows from Lemmas \ref{Lemm: Noise size}-b and \ref{Lemm: Noise size}-c
 that
 \begin{eqnarray*}
  ||\wh{\k}_i||_2  
   &\le& || \k_i ||_2 + || \wh{\G}^{-1}||_2 ||\r_i ||_2 \\
   &\le& \mu + \frac{2n (\mu + 1)}{\sigma -2n}.
 \end{eqnarray*}
 Since $n < \frac{1}{4}\sigma (1-\mu)$, we have $||\wh{\k}_i||_2 < 1$.
\end{proof}

The robustness of Algorithm \ref{Alg: ER}
for a noisy separable matrix
follows from the above discussion,
Corollary \ref{coro: active points of simplex under relaxed cond}, 
and Lemma \ref{Lemm: Check cond}.
\begin{theo}
 Let $\wt{\A}$ be a noisy separable matrix of the form $\A+\N$.
 Assume that Assumption \ref{Assump: Separability} holds 
 for the separable matrix $\A$ in $\wt{\A}$.
 Set $\epsilon = \frac{1}{4}\sigma(1-\mu)$ 
 where $\sigma = \sigma_r(\F)$ and $\mu = \mu(\K)$
 for the basis and weight matrices $\F$ and $\K$ of $\A$.
 If $||\N||_2 < \epsilon$,
 Algorithm \ref{Alg: ER} for $(\wt{\A}, \mbox{rank}(\A))$ 
 returns an index set $\IC$ such that  $||\wt{\A}(\IC) -\F||_2 < \epsilon$.
 \label{Theo: Robustness}
\end{theo}

In Theorem \ref{Theo: Robustness}, let $\F^* = \wt{\A}(\IC)$, 
and $\W^*$ be an optimal solution of the convex optimization problem,
\begin{equation*}
 \mbox{minimize} \ ||\wt{\A}(\IC)\X - \wt{\A}||_F^2 \ \mbox{subject to} \ \X \ge \0,
\end{equation*}
where the matrix $\X$ of size $r$-by-$m$ is the decision variable.
Then, $(\F^*, \W^*)$ serves as the NMF factor of $\wt{\A}$.
It is possible to evaluate 
the residual error of this factorization
in a similar way to the proof of Theorem 4 in \cite{Gil13}.
\begin{coro}
 Let $\w_i^*$  and $\wt{\a}_i$ be the column vectors of $\W^*$ and $\wt{\A}$, respectively.
 Then,  $||\F^* \w_i^* - \wt{\a}_i||_2 < 2 \epsilon$ for $i=1,\ldots,m$.
\end{coro}
\begin{proof}
 From Assumption \ref{Assump: Separability}-a,
 the column vectors $\w_i$ of $\W $ satisfy
 $||\w_i||_2 \le 1$ for $i=1, \ldots, m$.
 Therefore, for $i=1, \ldots, m$, 
 \begin{eqnarray*}
  || \F^* \w^*_i - \wt{\a}_i ||_2 
  &\le& || \F^* \w_i - \wt{\a}_i ||_2  \\
  &=& || \F^* \w_i  - \F\w_i + \F\w_i  - \a_i - \n_i  ||_2 \\
  &=& || (\F^* - \F) \w_i - \n_i  ||_2  \\
  &\le& ||\F^* - \F||_2 || \w_i ||_2 + || \n_i||_2  < 2 \epsilon, 
 \end{eqnarray*}
 where $\a_i$ and $\n_i$ denote the $i$th column vector of $\A$ and $\N$, respectively.
\end{proof}

\section{Implementation in Practice} 
\label{Sec: Implementation}
Theorem \ref{Theo: Robustness} guarantees that 
Algorithm \ref{Alg: ER} correctly identifies
the near-basis matrix of a noisy separable matrix
if the noise is smaller than some level.
But in the NMFs of matrices arising from practical applications,
it seems that the noise level  
would likely exceed the level
for which the theorem is valid.
In such a situation, 
the algorithm might generate more active points than hoped.
Therefore, we need to add a selection step in which 
$r$ points are selected from the active points.
Also, the number of active points  depends on 
which dimension we choose in the computation of the reduced matrix $\P$.
Algorithm \ref{Alg: ER} computes the reduced matrix $\P$ of the data matrix
and draws an origin-centered MVEE for the column vectors $\p_1, \ldots, \p_m$ of $\P$.
As we will see in Lemma \ref{Lemm: Number of active points},
the number of active points depends on the dimension of $\p_1, \ldots, \p_m$.
Therefore, we introduce an input parameter $\rho$ to control the dimension. 
By taking account of these considerations, 
we design a practical implementation of Algorithm \ref{Alg: ER}.

\begin{algorithm}
 \noindent 
 \caption{Practical Implementation of Algorithm \ref{Alg: ER}}
 \label{Alg: Practical ER}
 \textbf{Input:} $\M \in \Real^{d \times m}_+, r \in \Natural$, 
 and $\rho \in \Natural$. \\
 \textbf{Output:} $\IC$.
 \begin{enumerate}
  \item[\textbf{1:}] 
	       Run Algorithm \ref{Alg: ER} for $(\M, \rho)$.
	       Let $\JC$ be the index set returned by the algorithm.

  \item[\textbf{2:}]
	       If $|\JC| < r$, increase $\rho$ by 1 and go back to Step 1.
	       Otherwise, select $r$ elements from $\JC$
	       and construct the set $\IC$ of these elements.

 \end{enumerate}
\end{algorithm}

One may wonder whether Algorithm \ref{Alg: Practical ER} infinitely loops or not.
In fact, we can show that under some conditions, 
infinite loops do not occur.

\begin{lemm} \label{Lemm: Number of active points}
 For  $\p_1, \ldots, \p_m \in \Real^\rho$, 
 let $\SC = \{\pm \p_1, \ldots, \pm \p_m\}$.
 Suppose that Assumption \ref{Assump: MVEE} holds.
 Then,  $\Primal(\SC)$ has at least $\rho$ active points.
\end{lemm}
\begin{proof}
 Consider the KKT conditions (\ref{Eq: KKT}) for  $\Primal(\SC)$.
 Condition (\ref{Eq1: KKT}) requires 
 $\Omega(\z^*)$ to be nonsingular.
 Since $\mbox{rank}(\P) = \rho$ from the assumption,
 at least $\rho$ nonzero $z_i^*$ exist.
 Therefore,  we see from (\ref{Eq2: KKT})
 that $\Primal(\SC)$ has at least $\rho$ active points.
\end{proof}

\begin{prop} \label{Prop: Finite number of iterations}
Suppose that  we choose $r$ such that $r \le \mbox{rank}(\M)$.
 Then, Algorithm \ref{Alg: Practical ER} 
 terminates after a finite number of iterations.
\end{prop}
\begin{proof}
For the active index set $\JC$ constructed in Step 1, 
Lemma \ref{Lemm: Number of active points} guarantees that 
$|\JC| \ge \rho$.
The parameter $\rho$
increases  by 1 if $|\JC| < r$ in Step 2 
and can continue to increase up to  $\rho = \mbox{rank}(\M)$.
Since $r \le \mbox{rank}(\M)$, 
it is necessarily to satisfy $|\JC| \ge \rho \ge r$ 
 after a finite number of iterations.
\end{proof}

Proposition \ref{Prop: Finite number of iterations} 
implies that $\rho$ may not be an essential input parameter
since  Algorithm \ref{Alg: Practical ER} always terminates 
under $r \le \mbox{rank}(\M)$ even if starting with  $\rho=1$.

There are some concerns about Algorithm \ref{Alg: Practical ER}.
One is in how to select $r$ elements from an active index set $\JC$ in Step 2.
It is possible to have various ways to make the selection.
We rely on existing algorithms, such as XRAY and SPA, 
and perform these existing algorithms for $(\M(\JC), \rho)$.
Thus, Algorithm \ref{Alg: ER} can be regarded as a preprocessor 
which filters out basis vector candidates from the data points 
and enhance the performance of existing algorithms.
Another concern is in the computational cost of solving $\Primal$.
In the next section, we describe a cutting plane strategy 
for efficiently performing an interior-point algorithm.

\subsection{Cutting Plane Strategy for Solving $\Primal$}
\label{Subsec: Cutting plane}

Let $\SC$ be a set of $m$ points in $\Real^d$.
As mentioned in Section \ref{Sec: MVEE}, 
$O(m^3)$ arithmetic operations are required in
each iteration of an interior-point algorithm for $\Primal(\SC)$.
A cutting plane strategy is a way to reduce the number of points 
which we need to deal with in solving $\Primal(\SC)$.
The strategy was originally used in \cite{Sun04}. 
In this section, we describe the details of our implementation.

The cutting plane strategy for solving $\Primal$ has a geometric interpretation.
It is thought of that 
active points contribute a lot to the drawing 
the MVEE for a set of points
but inactive points make less of a contribution.
This geometric intuition can be justified by the following proposition.
Let $\L$ be a $d$-by-$d$ matrix.
We use the notation $\delta_{\LS}(\p)$ 
to denote $\langle \p \p^\top, \L \rangle$ for an element $\p \in \Real^d$ of $\SC$.
\begin{prop}
Let $\bar{\SC}$ be a subset of $\SC$.
If an optimal solution $\bar{\L}^*$  of $\Primal(\bar{\SC})$ satisfies 
$\delta_{\bar{\LS}^*}(\p) \le 1$ for all $p \in \SC \setminus \bar{\SC}$,
then $\bar{\L}^*$ is an optimal solution of $\Primal(\SC)$.
\end{prop}

The proof is omitted since it is obvious. 
The proposition implies that 
$\Primal(\SC)$ can be solved 
by using its subset $\bar{\SC}$ instead of $\SC$.
The cutting plane strategy offers a way of finding such a $\bar{\SC}$, 
in which a smaller  problem $\Primal(\bar{\SC})$ 
has the same optimal solution as $\Primal(\SC)$.
In this strategy, we first choose some points from $\SC$ and construct 
a set $\SC^1$ containing these points.
Let $\SC^k$ be the set constructed in the $k$th iteration.
In the $(k+1)$th iteration, we choose some points from
$\SC \setminus \SC^k$ and expand 
$\SC^k$ to $\SC^{k+1}$ by adding these points to $\SC^k$.
Besides expanding, 
we also shrink  $\SC^k$ by 
discarding some points which can be 
regarded as useless for drawing the origin-centered MVEE.
These expanding and shrinking phases play an important role 
in constructing a small set.
Algorithm \ref{Alg: Cutting plane} describes 
a cutting plane strategy for solving $\Primal(\SC)$.

\begin{algorithm}
 \noindent 
 \caption{Cutting Plane Strategy for Solving $\Primal(\SC)$} 
 \label{Alg: Cutting plane}
 \textbf{Input:}  $\SC = \{\p_1, \ldots, \p_m \}$. \\
 \textbf{Output:} $\L^*$.
 \begin{enumerate}
  \item[\textbf{1:}] 
	       Choose an initial set $\SC^1$ from $\SC$ and let $k=1$.

  \item[\textbf{2:}] 
	       Solve $\Primal(\SC^k)$ and find the optimal solution $\L^k$.
	       If  $\delta_{\LS^k}(\p) \le 1$ holds for 
	       all $\p \in \SC \setminus \SC^k$,  let $\L^* =\L^k$, and stop.
	       
  \item[\textbf{3:}]
	       Choose a subset $\FC$ of $\SC^k$ and  
	       a subset $\GC$ of 
	       $\{\p \in \SC \setminus \SC^k : \delta_{\LS^k}(\p) > 1 \}$.
	       Update $\SC^k$ as $\SC^{k+1} = (\SC^k \setminus \FC) \cup \GC$ 
	       and increase $k$ by $1$.
	       Then, go back to Step 2.

 \end{enumerate}
\end{algorithm}

Now, we give a description of our implementation of Algorithm \ref{Alg: Cutting plane}.
To construct the initial set $\SC^1$ in Step 1,
our implementation employs the algorithm used in \cite{Kum05, Tod07, Ahi08}.
The algorithm constructs a set $\SC^1$ by greedily choosing $2d$ points 
in a step-by-step manner
such that the convex hull is a $d$-dimensional crosspolytope 
containing as many points in $\SC$ as possible.
We refer the reader to  Algorithm 3.1 of \cite{Kum05} 
for the precise description.

To shrink and expand $\SC^k$ in Step 3,
we use a shrinking threshold parameter $\theta$ such that $\theta < 1$,
and an expanding size parameter $\eta$ such that $\eta \ge 1$.
These parameters are set before running Algorithm \ref{Alg: Cutting plane}.
For shrinking, 
we construct $\FC = \{\p \in \SC^k  : \delta_{\LS^k}(\p) \le \theta \}$ 
by using $\theta$.
For expanding, 
we arrange the points of $\{\p \in \SC \setminus \SC^k : \delta_{\LS^k}(\p) > 1 \}$ 
in descending order, as measured by  $\delta_{\LS^k}(\cdot)$,
and construct $\GC$ by choosing the top $(m-2d) / \eta$ points.
If the set $\{\p \in \SC \setminus \SC^k : \delta_{\LS^k}(\p) > 1 \}$ has less 
than $(m-2d) / \eta$ points, 
we choose all the points and construct $\GC$.

\section{Experiments} 
\label{Sec: Experiment}

We experimentally 
compared Algorithm \ref{Alg: Practical ER} with SPA and the variants of XRAY.
These two existing algorithms were chosen 
because their studies \cite{Bit12, Gil13b, Kum13} report that 
they outperform AGKM and Hottopixx, and scale to the problem size.
Two types of experiments were conducted:
one is the evaluation for the robustness 
of the algorithms to noise on synthetic data sets,
and the other is the application of the algorithms 
to clustering of real-world document corpora.

We implemented Algorithm \ref{Alg: Practical ER}, 
and  three variants of XRAY, ``max'', ``dist'' and ``greedy'', in MATLAB.
We put Algorithm \ref{Alg: Cutting plane} 
in Algorithm \ref{Alg: Practical ER} so it would
solve  $\Primal$ efficiently.
The software package SDPT3 \cite{Toh99b} was used for solving $\Primal(\SC^k)$ 
in Step 2 of Algorithm \ref{Alg: Cutting plane}.
The shrinking parameter $\theta$ 
and expanding size parameter $\eta$ were 
set as $0.9999$ and $5$, respectively.
The implementation of XRAY formulated the computation of 
the residual matrix  $\R = \A(\IC_k)\X^* - \A$ as a convex optimization problem,
\begin{equation*}
\X^* = \arg \min_{\XS \ge \zeros}||\A(\IC_k)\X - \A||_F^2.
\end{equation*}
For the implementation of SPA, 
we used code from the first author's website \cite{Gil13}.
Note that
SPA and XRAY are sensitive to 
the normalization of the column vectors of the data matrix
(\cite{Kum13}), and for this reason, 
we used a data matrix 
whose column vectors were not normalized.
All experiments were done in MATLAB on a 3.2 GHz CPU processor and 12 GB memory.

We will use the following abbreviations to represent the variants of algorithms.
For instance, 
Algorithm \ref{Alg: Practical ER} with SPA for an index selection of 
Step 2 is referred to as ER-SPA.
Also, the variant of XRAY with ``max'' selection policy
is referred to as XRAY(max).

\subsection{Synthetic Data} \label{Subsec: Artificial data}
Experiments were conducted for 
the purpose of seeing how well Algorithm \ref{Alg: Practical ER} could improve 
the robustness of SPA and XRAY to noise.
Specifically, 
we compared it with SPA, XRAY(max), XRAY(dist), and XRAY(greedy).
The robustness of algorithm was measured by a recovery rate.
Let $\IC$ be an index set of basis vectors in a noisy separable matrix,
and $\IC^*$ be an index set returned by an algorithm.
The recovery rate is the ratio given by $|\IC \cap \IC^*| \ / \ |\IC|$.

We used synthetic data sets of the form $\F(\I, \K)\Pib + \N$ 
with $d=250$, $m=5,000$, and $r=10$.
The matrices $\F, \K, \Pib$ and $\N$ were synthetically generated as follows.
The entries of $\W \in \Real^{d \times r}_+$ 
were drawn from a uniform distribution on the interval $[0,1]$.
The column vectors of $\K \in \Real^{r \times \ell}_+$ were from a Dirichlet distribution
whose $r$ parameters were uniformly  from the interval $[0,1]$.
The permutation matrix $\Pib$ was randomly generated.
The entries of the noise matrix $\N \in \Real^{d \times m}$ were from
a normal distribution with mean $0$ and standard deviation $\delta$.
The parameter $\delta$ determined 
the intensity of the noise, and 
it was chosen from $0$ to $0.5$ in $0.01$ increments.
A single data set consisted of 51 matrices with various amounts of noise,
and we made 50 different data sets.
Algorithm \ref{Alg: Practical ER} was performed in the setting that 
$\M$ is a matrix in the data set and $r$ and $\rho$ are each $10$.

\begin{figure}[p]
 \begin{center}
\begin{minipage}{.47\linewidth}
 \includegraphics[width=\linewidth]{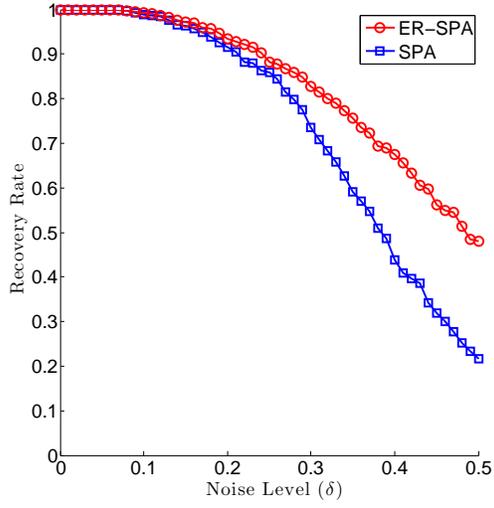}
\end{minipage}
\begin{minipage}{.47\linewidth}
 \includegraphics[width=\linewidth]{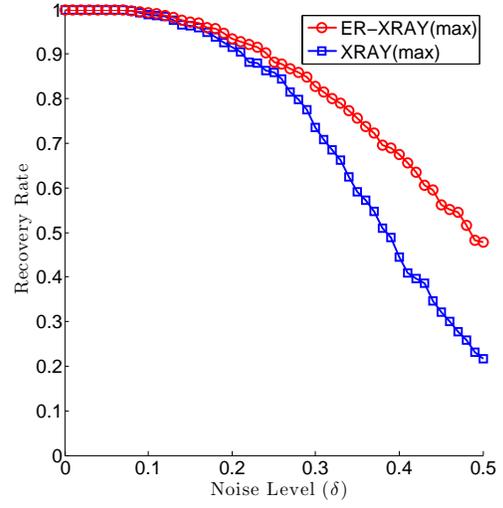}
\end{minipage}
\begin{minipage}{.47\linewidth}
 \includegraphics[width=\linewidth]{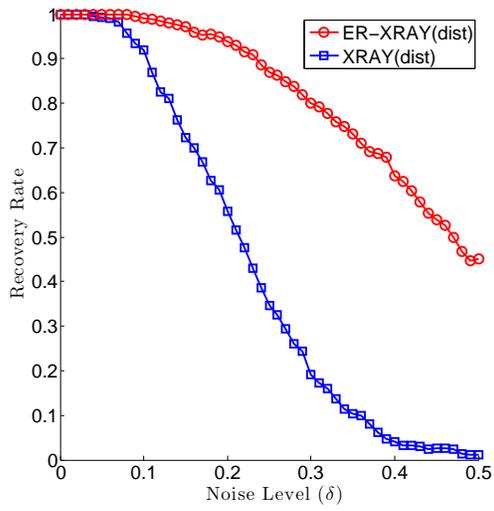}
\end{minipage}
\begin{minipage}{.47\linewidth}
 \includegraphics[width=\linewidth]{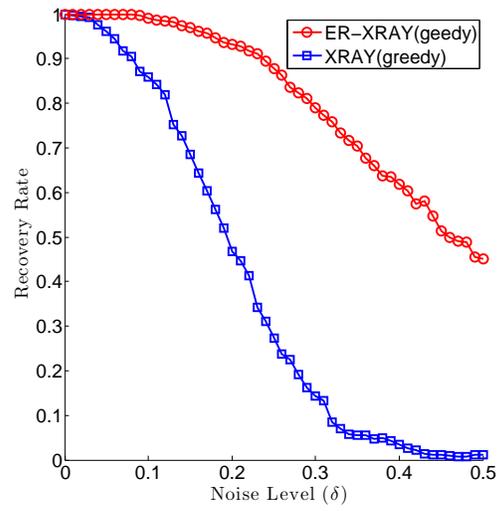}
\end{minipage}
\end{center}
\caption{Comparison of the recovery rates of Algorithm \ref{Alg: Practical ER}
 with SPA and XRAY.}
 \label{Fig: Cmp}
\end{figure}

\begin{table}[h]
 \centering
 \caption{Maximum values of noise level $\delta$ for different recovery rates in percentage.}
  \label{Tab: Max noise level}
 \begin{tabular}{|r|cccc|}
  \hline
   Recovery rate    & 100\% & 90\% & 80\% & 70\% \\
  \hline 
  ER-SPA          &  0.06 & 0.24 & 0.32  & 0.37 \\
  SPA             &  0.05 & 0.21 & 0.27  & 0.31 \\
  \hline
  ER-XRAY(max)    &  0.06 & 0.24 & 0.32 & 0.37 \\
  XRAY(max)       &  0.05 & 0.21 & 0.27 & 0.31\\
  \hline
  ER-XRAY(dist)   &  0.07 & 0.23 & 0.29 & 0.36 \\
  XRAY(dist)      &  0.03 & 0.10 & 0.13 & 0.16 \\
  \hline
  ER-XRAY(greedy) & 0.07  & 0.23 & 0.29 & 0.35 \\
  XRAY(greedy)    & 0.00  & 0.08 & 0.12 & 0.14 \\
  \hline
 \end{tabular}
\end{table}

Figure \ref{Fig: Cmp} depicts the average recovery rate 
on the 50 data sets for Algorithm \ref{Alg: Practical ER}, SPA and XRAY.
Table \ref{Tab: Max noise level} summarizes the maximum values of
noise level $\delta$ for different recovery rates in percentage.
The noise level was measured by $0.01$, 
and hence, for instance, the entry ``$0.00$'' at XRAY(greedy) for 100\% recovery rate 
means that the maximum value is  in the interval $[0.00, 0.01)$.
We see from the figure that Algorithm \ref{Alg: Practical ER}  improved
the recovery rates of the existing algorithms.
In particular, 
the recovery rates of XRAY(dist) and XRAY(greedy)
rapidly decrease as the noise level increases,
but Algorithm \ref{Alg: Practical ER} significantly improved them.
Also, the figure shows that
Algorithm \ref{Alg: Practical ER} tended 
to slow the decrease in the recovery rate.
We see from the table that Algorithm \ref{Alg: Practical ER} is more robust to noise 
than SPA and XRAY.

\begin{table}[h]
 \centering
 \caption{Average  number of active points and 
  elapsed time of Algorithm \ref{Alg: Practical ER}}
  \label{Tab: Stat of time and numActPts}
 \begin{tabular}{|c|c|cccc|}
  \hline
  $\delta$ &  Active points & \multicolumn{4}{c|}{Elapsed time (second)} \\
      &    & ER-SPA & ER-XRAY(max) & ER-XRAY(dist) & ER-XRAY(greedy) \\
  \hline
  0     & 10 & 1.05 & 1.07 & 1.07 & 1.07 \\
  0.25  & 12 & 3.08 & 3.10 & 3.10 & 3.10 \\
  0.5   & 23 & 4.70 & 4.71 & 4.71 & 4.71 \\
  \hline
 \end{tabular}
\end{table}

Table \ref{Tab: Stat of time and numActPts} summarizes 
the average number of active points and elapsed time for 50 data sets
taken by Algorithm \ref{Alg: Practical ER} with $\delta=0, 0.25$ and $0.5$.
We read from the table that
the elapsed time increases with the number of active points.
The average elapsed times of SPA, XRAY(max), XRAY(dist), and XRAY(greedy) 
was respectively $0.03$, $1.18$, $16.80$ and $15.85$ in seconds.
Therefore, we see that the elapsed time of Algorithm \ref{Alg: Practical ER} was
within a reasonable range.

\subsection{Application to Document Clustering}
\label{Subsec: Document clustering}

Consider a set of $d$ documents. 
Let $m$ be the total number of words appearing in the document set.
We represent the documents by a bag-of-words.
That is, the $i$th document is represented as an $m$-dimensional
vector $\a_i$, whose elements are the appearance frequencies of words in the document.
A document vector $\a_i$ can be assumed to be generated 
by a convex combination of several topic vectors $\w_1, \ldots \w_r$.
This type of generative model has been used in many papers, 
for instance, \cite{Xu03, Sha06, Aro12b, Aro13, Din13, Kum13}.

Let $\W$ be an $r$-by-$m$ topic matrix such that $\w_1^\top, \ldots, \w_r^\top$ 
are stacked from  top to bottom and are of the form $(\w_1; \ldots; \w_r)$.
The model allows us to write a document vector in the form
$\a_i^\top = \f_i^\top \W$ by using a coefficient vector $\f_{i} \in \Real^r$ 
such that $\e^\top \f_i = 1$ and $\f_i \ge \0$.
This  means that we have $\A = \F\W$ for 
a document-by-word matrix $\A = (\a_1; \ldots; \a_d) \in \Real^{d \times m}_+$, 
a coefficient matrix $\F = (\f_1; \ldots; \f_d) \in \Real^{d \times r}_+$, and 
a topic matrix $\W = (\w_1; \ldots; \w_r) \in \Real^{r \times m}_+$.
In the same way as is described in  \cite{Aro12b,Aro13, Din13, Kum13},
we assume that a document-by-word matrix $\A$ is separable.
This requires that $\W$ is of $(\I,\K)\Pib$, and it means that
each topic has an {\it anchor word}.
An anchor word is a word that 
is contained in one topic but not contained in the other topics.
If an anchor word is found, 
it suggests that the associated topic exists.

Algorithms for Problem \ref{Prob: SepNMF} can be used for clustering documents 
and finding topics for the above generative model.
The algorithms for a document-word matrix $\A$ return an index set $\IC$.
Let $\F = \A(\IC)$. 
The row vector elements $f_{i1}, \ldots, f_{ir}$ of $\F$
can be thought of as the contribution rate 
of topics $\w_1, \ldots, \w_r$ for generating a document $\a_i$.
The highest value $f_{ij^*}$ among the  elements
implies that the topic $\w_{j^*}$ 
contributes the most to the generation of document $\a_i$.
Hence, we assign document $\a_i$ to a cluster having the topic $\w_{j^*}$.
There is an alternative to using $\F$ for measuring
the contribution rates of the topics.
Step 1 of Algorithm \ref{Alg: ER} produces 
a rank-$r$ approximation matrix $\A^r$ to $\A$ as a by-product.
Let $\F' = \A^r(\IC)$, and use it as an alternative to $\F$.
We say that clustering with $\F$ is clustering with the original data matrix, 
and that clustering with $\F'$ is  clustering with a low-rank approximation data matrix.

Experiments were conducted in 
the purpose of investigating clustering performance of algorithms
and also checking whether meaningful topics could be extracted.
To investigate the clustering performance, we used only SPA 
since our experimental results implied that XRAY would underperform.
We assigned the values of the document-word matrix on the basis 
of the tf-idf weighting scheme, for which we refer the reader to \cite{Man08}, 
and normalized the row vectors to the unit 1-norm.

To evaluate the clustering performance, 
we measured the accuracy (AC) and normalized mutual information (NMI).
These measures are often used for this purpose (See, for instance, \cite{Xu03, Man08}).
Let $\Omega_1, \ldots, \Omega_r$ be the manually classified classes
and $\CC_1, \ldots, \CC_r$ be the clusters constructed by an algorithm.
Both $\Omega_i$ and $\CC_j$ are the subsets 
of the document set $\{\a_1, \ldots, \a_m\}$ such that 
each subset does not share any documents and the union of all subsets 
coincides with the document set.
AC is computed as follows.
First, compute the correspondence between 
classes $\Omega_1, \ldots, \Omega_r$
and clusters $\CC_1, \ldots, \CC_r$ 
such that the total number of common documents $\Omega_i \cap \CC_j$ is maximized.
This computation can be done by solving an assignment problem.
After that, rearrange the classes and clusters in the obtained order and compute 
\begin{equation*}
 \frac{1}{d} \sum_{k=1}^r |\Omega_k \cap \CC_k|.
\end{equation*}
This value is the AC for the clusters constructed by an algorithm.
NMI is computed as
\begin{equation*}
 \frac{I(\Omega, \CC)}{\frac{1}{2}  (E(\Omega) + E(\CC))}.
\end{equation*}
$I$ and $E$ denote the mutual information and entropy
for the class family $\Omega$ and cluster family $\CC$ where 
$\Omega = \{\Omega_1, \ldots, \Omega_r\}$ and $\CC = \{\CC_1, \ldots, \CC_r\}$.
We refer the reader to Section 16.3 of \cite{Man08} 
for the precise forms of $I$ and $E$.

Two document corpora were used for 
the clustering-performance evaluation:
Reuters-21578 and 20 Newsgroups.
These corpora are publicly available from 
the UCI Knowledge Discovery in Databases Archive
\footnote{\url{http://kdd.ics.uci.edu}}.
In particular, we used the data preprocessing 
of Deng Cai \footnote{\url{http://www.cad.zju.edu.cn/home/dengcai/}},
in which multiple classes are discarded.
The Reuters-21578 corpus consists of 21,578 documents 
appearing in the Reuters newswire in 1987, 
and these documents are manually classified into 135 classes.
The text corpus is reduced by the preprocessing to 8,293 documents in 65 classes.
Furthermore, we cut off classes with less than 5 documents.
The resulting corpus contains 8,258 documents with 18,931 words 
in 48 classes, and the sizes of the classes range from 5 to 3,713.
The 20 Newsgroups corpus consists of 18,846 documents with 26,213 words 
appearing in 20 different newsgroups. 
The size of each class is about 1,000.

We randomly picked some classes from the corpora
and evaluated the clustering performance 50 times.
Algorithm \ref{Alg: Practical ER} was performed
in the setting that 
$\M$ is a document-word matrix and 
$r$ and $\rho$ each are the number of classes.
In clustering with a low-rank approximation data matrix,
we used the rank-$r$ approximation matrix to a document-word matrix.

\begin{table}[h]

 \centering
  \caption{(Reuters-21578) Average AC and NMI of ER-SPA and SPA 
  with the original data matrix and low-rank approximation data matrix. }
  \label{Tab: Clustering for Reuters}
 \begin{tabular}{|c|cc|cc|cc|cc|}
  \hline     
  & \multicolumn{4}{c|}{AC} & \multicolumn{4}{c|}{NMI} \\
  \hline    
  & \multicolumn{2}{c|}{Original} & \multicolumn{2}{c|}{Low-rank approx.} 
  & \multicolumn{2}{c|}{Original} & \multicolumn{2}{c|}{Low-rank approx.} \\
   \# Classes  & ER-SPA & SPA & ER-SPA  & SPA & ER-SPA & SPA & ER-SPA  & SPA \\
  \hline
    6  & 0.605 & 0.586  & 0.658 & 0.636 
       & 0.407 & 0.397  & 0.532 & 0.466 \\

    8  & 0.534 & 0.539  & 0.583 & 0.581
       & 0.388 & 0.387  & 0.491 & 0.456 \\

    10 & 0.515 & 0.508  & 0.572 & 0.560
       & 0.406 & 0.393  & 0.511 & 0.475 \\

   12 & 0.482 & 0.467  & 0.532 & 0.522 
      & 0.399 & 0.388  & 0.492 & 0.469 \\

  \hline
 \end{tabular}

 \centering
  \caption{(20 Newsgroups) Average AC and NMI of ER-SPA and SPA 
  with the original data matrix and low-rank approximation data matrix.}
  \label{Tab: Clustering for 20 Newsgroups}
 \begin{tabular}{|c|cc|cc|cc|cc|}
  \hline     
  & \multicolumn{4}{c|}{AC} & \multicolumn{4}{c|}{NMI} \\
  \hline    
  & \multicolumn{2}{c|}{Original} & \multicolumn{2}{c|}{Low-rank approx.} 
  & \multicolumn{2}{c|}{Original} & \multicolumn{2}{c|}{Low-rank approx.} \\
   \# Classes  & ER-SPA & SPA & ER-SPA  & SPA & ER-SPA & SPA & ER-SPA  & SPA \\
  \hline
    6  & 0.441 & 0.350 & 0.652 & 0.508 
       & 0.314 & 0.237 & 0.573 & 0.411 \\

    8  & 0.391 & 0.313 & 0.612 & 0.474 
       & 0.306 & 0.242 & 0.555 & 0.415 \\

    10 & 0.356 & 0.278 & 0.559 & 0.439
       & 0.291 & 0.228 & 0.515 & 0.397 \\

   12 & 0.319 & 0.240 & 0.517 & 0.395 
      & 0.268 & 0.205 & 0.486 & 0.372  \\

  \hline
 \end{tabular}
\end{table}

Tables \ref{Tab: Clustering for Reuters} and \ref{Tab: Clustering for 20 Newsgroups}
show the results for Reuters-21578 and 20 Newsgroups, respectively.
They summarize the average ACs and NMIs of ER-SPA and SPA.
The column with ``\# Classes'' lists the number of classes we chose.
The columns labeled ``Original'' and ``Low-rank approx.''
are respectively the averages of the corresponding 
clustering measurements with the original data matrix 
and low-rank approximation data matrix.
The tables suggest that 
clustering with a low-rank approximation data matrix 
performed better than clustering with the original data matrix.
We see from Table \ref{Tab: Clustering for Reuters} that 
ER-SPA could achieve improvements in 
the AC and NMI of SPA on Reuters-21578 
when the clustering was done with a low-rank approximation data matrix.
Table \ref{Tab: Clustering for 20 Newsgroups} indicates that
ER-SPA outperformed SPA in AC and NMI on 20 Newsgroups.

Finally, 
we compared the topics obtained by ER-SPA and SPA.
We used the BBC corpus in \cite{Gre06}, which 
is publicly available from the website
\footnote{\url{http://mlg.ucd.ie/datasets/bbc.html}}.
The documents in the corpus have been 
subjected by preprocessed such as 
stemming, stop-word removal, and low word frequency filtering.
It consists of 2,225 documents with 9,636 words
that appeared on the BBC news website in 2004-2005.
The documents were news on 5 topics:
``business'', ``entertainment'', ``politics'', ``sport'' and ``tech''.

\begin{table}[h]
 \centering
  \caption{AC and NMI of ER-SPA and SPA 
  with low-rank approximation data matrix for BBC.}
  \label{Tab: Clustering for BBC}
 \begin{tabular}{|cc|cc|}
  \hline     
  \multicolumn{2}{|c|}{AC} & \multicolumn{2}{c|}{NMI} \\
  \hline
   ER-SPA & SPA  &  ER-SPA & SPA  \\
  \hline
    0.939 &  0.675    &  0.831  &  0.472 \\ 
  \hline
 \end{tabular}

  \centering
 \caption{Anchor words and top-5 frequent words in
  topics grouped by ER-SPA and SPA for BBC}
  \label{Tab: Topics for BBC}
  \begin{tabular}{|c|cccccc|}
  \hline
  & Anchor word & 1 & 2 &  3 & 4 & 5  \\
  \hline 
   ER-SPA & film & award & best & oscar & nomin & actor  \\
   SPA & film & award & best & oscar & nomin & star \\ 
   \hline
   ER-SPA & mobil & phone & user & softwar & microsoft & technolog \\
   SPA & mobil & phone & user & microsoft &  music & download \\
   \hline
   ER-SPA & bank & growth & economi & price & rate & oil \\
   SPA & bank & growth & economi & price & rate & oil \\
   \hline
   ER-SPA & game & plai & player & win & england & club \\
   SPA & fiat & sale & profit & euro &  japan & firm \\
   \hline
   ER-SPA & elect & labour & parti & blair & tori & tax \\
   SPA & blog & servic & peopl & site & firm & game \\
  \hline
  \end{tabular}
 \end{table}

Table \ref{Tab: Clustering for BBC} shows the ACs and NMIs 
of ER-SPA and SPA on the low-rank approximation data matrix 
for the BBC corpus.
The table indicates that 
the AC and NMI of ER-SPA are higher than those of SPA.
Table \ref{Tab: Topics for BBC} summarizes the words 
in the topics obtained by ER-SPA and SPA.
The topics were computed by using a low-rank approximation data matrix.
The table lists the anchor word and the 5 most frequent words in each topic
from left to right.
We computed the correspondence between topics obtained by ER-SPA and SPA
and grouped the topics for each algorithm.
Concretely, we measured the 2-norm of each topic vector
and computed the correspondence by solving an assignment problem.
We can see from the table that 
the topics obtained by these two algorithms are almost the same 
from the first to the third panel, 
and they seem to correspond to ``entertainment'', ``tech'' and ``business''.
The topics in the fourth and fifth panels, however, are different.
The topic in the fifth panel by ER-SPA 
seems to correspond to  ``politics''.
In contrast, it is difficult to find the topic 
corresponding to ``politics'' in the panels by SPA.
These show that ER-SPA could extract more recognizable topics than SPA.

\begin{remark}
 Sparsity plays an important role in computing the SVD for a large document corpus.
 In general, a document-word matrix arising from a text corpus is quite sparse.
 Our implementation of Algorithm \ref{Alg: Practical ER} 
 used the MATLAB command {\tt svds} that
 exploits  the sparsity of a matrix in the SVD computation.
 The implementation could work on all data of 20 Newsgroups corpus, 
 which formed a document-word matrix of size 18,846-by-26,213.
\end{remark}

\section{Concluding Remarks}
\label{Sec: Concluding}

We presented Algorithm \ref{Alg: ER} for Problem \ref{Prob: SepNMF}
and formally showed that it has correctness and robustness properties.
Numerical experiments on synthetic data sets demonstrated 
that Algorithm \ref{Alg: Practical ER}, 
which is the practical implementation of Algorithm \ref{Alg: ER},
is robustness to noise.
The robustness of the algorithm was measured in terms of the recovery rate.
The results indicated that Algorithm \ref{Alg: Practical ER} 
can improve the recovery rates of SPA and XRAY.
The algorithm was then applied to document clustering.
The experimental results implied 
that it outperformed SPA and extracted more recognizable topics.

We will conclude by suggesting a direction for future research.
Algorithm \ref{Alg: Practical ER} needs to do two computations:
one is the SVD of the data matrix and the other is the MVEE for
a set of reduced-dimensional data points.
It would be ideal to have a single computation 
that could be parallelized.
The MVEE computation requires that the convex hull of data points is full-dimensional.
Hence, the SVD computation should be carried out on data points.
However, if we could devise an alternative convex set for MVEE,
it would possible to avoid SVD computation.
It would be interesting to investigate 
the possibility of algorithms that 
find near-basis vectors by using the other type of convex set for data points.

\section*{Acknowledgments} 
The author would like to thank Akiko Takeda of the University of Tokyo
for her insightful and enthusiastic discussions, and 
thank the referees for careful reading and helpful suggestions
that considerably improved the presentation of this paper.

\bibliographystyle{plain}
\bibliography{M13}

\end{document}